\documentclass[10pt,twocolumn,twoside] {IEEEtran}		

\usepackage{hyperref}
\hypersetup{
	pdftitle={Group-Sparse Signal Denoising: Non-Convex Regularization, Convex Optimization}, 
	pdfauthor={Po-Yu Chen and Ivan Selesnick},
	colorlinks=true,
	linkcolor = black,
	citecolor = black,
	breaklinks = true,
}

\usepackage{graphicx}
\usepackage{epstopdf}
\usepackage{amssymb}
\usepackage{mathtools}
\usepackage{amsthm}
\usepackage{booktabs}
\usepackage{cite}
\usepackage{animate}
\usepackage{array}

\usepackage{mathtools}
\usepackage{array}

\providecommand{\norm}[1]{\lVert#1\rVert}

\DeclarePairedDelimiter{\abs}{\lvert}{\rvert}			

\newcommand{\inv}{^{-1}}

\newcommand{\lam}{{\lambda} }
\newcommand{\half}{\frac{1}{2}}

\newcommand{\eps}{\epsilon}

\newcommand{\e}{\mathbf{e}}

\newcommand{\s}{\mathbf{s}}

\newcommand{\bv}{\mathbf{v}}
\newcommand{\w}{\mathbf{w}}
\newcommand{\x}{\mathbf{x}}
\newcommand{\y}{\mathbf{y}}

\newcommand{\A}{\mathbf{A}}

\newcommand{\0}{\mathbf{0}}

\newcommand{\calN}{\mathcal{N}}

\newcommand{\RR}{\mathbb{R}}
\newcommand{\ZZ}{\mathbb{Z}}

\newcommand{\opt}{^{\ast}}

\newcommand{\tr}{^{T}}					

\newcommand{\reg}{\phi}					

\newcommand{\iter}[1]{^{(#1)}}

\DeclareMathOperator{\std}{std}

\newcommand{\prl}{\RR_+^{\ast}}			
\newcommand{\nnrl}{\RR_+}			

\DeclareMathOperator{\tf}{\theta}			

\newcommand{\slog}{_{\mathrm{log}}}
\newcommand{\satan}{_{\mathrm{atan}}}
\newcommand{\srat}{_{\mathrm{rat}}}

\newcommand{\NZ}{S}		

\newcommand{\RNZ}{\RR\!\setminus\!\{0\}}

\DeclareMathOperator{\ogs}{ogs}
\DeclareMathOperator{\STFT}{STFT}

\theoremstyle{definition}
\newtheorem{prop}{Proposition}
\newtheorem{cor}{Corollary}
\newtheorem{thm}{Theorem}
\newtheorem{lemma}{Lemma}

\renewcommand{\leq}{\leqslant}
\renewcommand{\geq}{\geqslant}

\title{Group-Sparse Signal Denoising: Non-Convex Regularization, Convex Optimization}

\author{Po-Yu Chen and Ivan W. Selesnick
\thanks{The authors are with the Department of Electrical and Computer Engineering, 
Polytechnic Institute of New York University, 6 Metrotech Center, Brooklyn, NY 11201.
Email: poyupaulchen@gmail.com, selesi@poly.edu, phone: 718 260-3416, fax: 718 260-3906.}%
\thanks{This research was support by the NSF under Grant No. CCF-1018020.}%
}

\begin{document}
\maketitle

\begin{abstract}

Convex optimization with sparsity-promoting convex regularization is a standard approach for estimating sparse signals in noise.
In order to promote sparsity more strongly than convex regularization, it is also standard practice to employ non-convex optimization.
In this paper, we take a third approach.
We utilize a non-convex regularization term chosen such that the
total cost function (consisting of data consistency and regularization terms) is convex. 
Therefore, sparsity is more strongly promoted than in the standard convex formulation, 
but without sacrificing the attractive aspects of convex optimization (unique minimum, robust algorithms, etc.).
We use this idea to improve the recently developed `overlapping group shrinkage' (OGS) algorithm
for the denoising of group-sparse signals.
The algorithm is applied to the problem of speech enhancement with favorable results in
terms of both SNR and perceptual quality.

\end{abstract}

\section{Introduction}

In this work, we address the problem of estimating a vector $\x$ from an observation $\y$,
\begin{equation}
	y(i) = x(i) + w(i),
	\ i \in \ZZ_N = \{0, \dots, N-1\},
\end{equation}
where $\w$ is additive white Gaussian noise (AWGN).
We assume that $\x$ is a group-sparse vector.
By group-sparse, we mean that large magnitude values of $\x$ tend not to be isolated.
Rather, large magnitude values tend to form clusters (groups).
Furthermore, we do not assume that the group locations are known, nor that the group boundaries are known.
In fact, we do not assume that the groups have well defined boundaries.
An example of such a vector (in 2D) is the spectrogram of a speech waveform.
The spectrogram of a speech waveform exhibits areas and ridges of large magnitude, but not isolated large values.
The method proposed in this work will be demonstrated on the problem of speech filtering.

Convex and non-convex optimization are both common practice 
for the estimation of sparse vectors from noisy data \cite{Bach_2012_now}.
In both cases one often seeks the solution $\x\opt \in \RR^N$ to the problem
\[
	\x\opt = \arg \min_{\x} \Bigl\{  F(\x) = \half \norm{\y - \x}_2^2 + \lambda R(\x) \Bigr\}
\]
where $ R(\x) : \RR^N \to \RR $ is the regularization (or penalty) term
and $ \lam > 0 $.
Convex formulations are advantageous in that a wealth of convex optimization theory can be leveraged and robust algorithms with guaranteed convergence are available \cite{Boyd_convex}.
On the other hand, non-convex approaches are advantageous in that they usually yield sparser solutions for a given residual energy.
However, non-convex formulations are generally more difficult to solve (due to suboptimal local minima, initialization issues, etc.).
Also, solutions produced by non-convex formulations are generally discontinuous functions of input data (e.g., the discontinuity of the hard-threshold function). 

Generally, convex approaches are based on sparsity-promoting convex penalty functions (e.g., the $\ell_1$ norm), while non-convex approaches are based on non-convex penalty functions (e.g., the $\ell_p$ pseudo-norm with $p < 1$ \cite{Lorenz_2007}, re-weighted $\ell_2$/$\ell_1$ \cite{Candes_2008_JFAP, Wipf_2010_TSP}).
Other non-convex algorithms seek sparse solutions directly
(e.g., OMP \cite{Mallat_1998},
iterative hard thresholding \cite{Blumensath_2012_SP, Kingsbury_2003_ICIP,Portilla_2007_SPIE, Foucart_2010_SIAM},
and
greedy $\ell_1$ \cite{Kozlov_2010_geomath}).

In this work, we take a different approach, proposed by Blake and Zimmerman \cite{Blake_1987} and 
by Nikolova \cite{Nikolova_1998_ICIP}.
Namely, the use of a non-convex non-smooth penalty function chosen such that the
total cost function $F$ (consisting of data consistency and regularization terms) is strictly convex.
This is possible ``by balancing the positive second derivatives in the [data consistency term] against the negative second derivatives in the [penalty] terms'' \cite[page 132]{Blake_1987}.
This idea has been further extended by Nikolova et al.\ \cite{Nikolova_1998_TIP, Nikolova_1999_TIP, Nikolova_2008_SIAM, Nikolova_2010_TIP}.

The contribution of this work relates to (1) the formulation of the group-sparse denoising problem as a convex optimization problem albeit defined in terms of a non-convex penalty function, and (2) the derivation of a computationally efficient iterative algorithm that monotonically reduces the cost function value.
We utilize non-convex penalty functions (in fact, concave on the positive real line) with parametric forms;
and we identify an interval for the parameter that ensures the strict convexity of the total cost function, $F$.
As the total cost function is strictly convex, the minimizer is unique and can be obtained reliably using convex optimization techniques. 
The algorithm we present is derived according to the principle of majorization-minimization  (MM) \cite{FBDN_2007_TIP}.
The proposed approach:
\begin{enumerate}
\item
promotes sparsity more strongly than any convex penalty function can,
\item
is translation invariant (due to groups in the proposed method being fully overlapping),
\item
is computationally efficient ($O(N)$ per iteration) with decreasing cost function,
and
\item
requires no algorithmic parameters (step-size, Lagrange, etc.).
\end{enumerate}
We demonstrate below that the proposed approach substantially improves upon
our earlier work that considered only convex regularization \cite{Chen_2014_OGS}.

\subsection{Related Work}

The estimation and reconstruction of signals with group sparsity properties has been addressed by numerous authors.
We make a distinction between two cases:
non-overlapping groups \cite{Yuan_2006, Kowalski_2009_SIVP, Kowalski_2009_ACHA, Eldar_2009_Tinfo, Chartrand_2013_ICASSP}
and overlapping groups
\cite{Jenatton_2011_MLR, Deng_2011_techrep, Figueiredo_2011_SPARS, Jacob_2009_cnf, Bayram_2011_picassp, Bayram_2013_SIVP_CoSep_preprint, Peyre_2011_eusipco, Bach_2012_now, Mosci_2010_NIPS, Yuan_2011_NIPS, Chen_2012_AAS}.
The non-overlapping case is the easier case: when the groups are non-overlapping, there is a decoupling of variables, which simplifies the optimization problem. When the groups are overlapping, the variables are coupled.
In this case, it is common to define auxiliary variables (e.g., through the variable splitting technique)
and apply methods such as the alternating direction method of multipliers (ADMM) \cite{Boyd_2011_admm}.
This approach increases the number of variables (proportional to the group size) and hence increases memory usage and data indexing.
In previous work we describe the `overlapping group shrinkage' (OGS) algorithm \cite{Chen_2014_OGS} for the overlapping-group case that does not use auxiliary variables.
The OGS algorithm exhibits favorable asymptotic convergence in comparison with algorithms that use auxiliary variables \cite[Fig.\ 5]{Chen_2014_OGS}.
In comparison with previous work on convex optimization for overlapping group sparsity, including \cite{Chen_2014_OGS},
the current work promotes sparsity more strongly.
The current work extends the OGS algorithm to the case of non-convex regularization,
yet remains within the convex optimization framework.

As noted above,
the balancing of the data consistency term and the penalty term, so as to formulate a convex problem with a non-convex penalty term, was described in Refs.~\cite{Blake_1987, Nikolova_1998_ICIP} and extended in \cite{Nikolova_1998_TIP, Nikolova_1999_TIP, Nikolova_2008_SIAM, Nikolova_2010_TIP}.
This approach was used to initialize a scheme named `graduated non-convexity' (GNC) in \cite{Blake_1987}.
The goal of GNC is to minimize a non-convex function $F$ by minimizing a sequence of functions $F_k, \, k \geq 1$.
The first one is a convex approximation of $F$, and the subsequent ones are non-convex and progressively similar to $F$.
In order that the initial approximation of $F$ be convex, the penalty function must satisfy an eigenvalue condition \cite{Blake_1987}.
A looser condition, which promotes sparsity more strongly, can be expressed as a semidefinite program (SDP), but this incurs a higher computational cost \cite{Selesnick_2013_arXiv_MSC}.
In the method described here, we use the same balancing idea as in GNC; however, our goal is to minimize a convex function, not a non-convex one as in GNC.
In particular, we use the balancing idea to construct a convex function that maximally promotes sparsity, and we seek to subsequently solve this convex problem.
We note that here our primary goal is to capture group sparsity behavior, which is not considered in the GNC work.
We also note that the computationally demanding SDP arising in Ref.~\cite{Selesnick_2013_arXiv_MSC} does not arise in the current work.
The algorithm developed here is computationally simple.

\section{Preliminaries}

\subsection{Notation}

We will work with finite-length discrete signals which we denote in 
lower case bold.
The $N$-point signal $\x$ is written as
 \[
 	\x = [x(0), \dots, x(N-1)] \in \RR^N.
\]
We use the notation
\begin{equation}
	\x_{i,K} = [x(i), \, \dots, x(i+K-1)] \in \RR^K
\end{equation}
to denote the $i$-th group of vector $\x$ of size $K$.
We consistently use $ K $ (a positive integer) to denote the group size.
At the boundaries (i.e., for $i < 0$ and $i>N-K$), some indices of $ \x_{i,K} $ fall
outside $ \ZZ_N $.
We take these values as zero; i.e., 
for $ i \notin \ZZ_N $, we take  $ x(i) = 0 $.

We denote the
non-negative real line as
$	\nnrl := \{ x \in \RR \, : \, x \geq 0\} $
and
the positive real line as
$	\prl := \{ x \in \RR \, : \, x > 0\} $.
Given a function $ f : \RR \to \RR $, 
the left-sided and right-sided derivatives of $ f $ at $ x $ are denoted $f'(x^-)$ and $f'(x^+)$, respectively.
The notation $ A \!\setminus\! B $ denotes set difference; i.e., $ A \!\setminus\! B = \{ a \in A : a \notin B \} $.

\subsection{Penalty Functions}
\label{sec:pen}

We will make the following 
assumptions on the penalty function, $\reg : \RR \to \RR $.
\begin{enumerate}
\item
$ \reg $ is continuous on $ \RR $.
\item
$ \reg $ is twice differentiable on $ \RNZ $. 
\item
$ 	\reg(-x)  = \reg(x) $  (symmetric)
\item
$	\reg'(x) > 0, \  \forall x > 0 	$  (increasing on $\prl$)
\item
$	\reg''(x) \leq 0, \  \forall x > 0  $   (concave on $\prl$)
\item
$ \reg'(0^+) = 1 $ (unit slope at zero)
\item
$ \reg''(0^+) \leq \reg''(x), \; \forall x > 0 $ (maximally concave at zero)
\end{enumerate}

We will utilize penalty functions parameterized by a scalar parameter, $ a >  0 $.
We use the notation $ \reg(x; a) $ to denote the parameterized form.

Examples of parameterized penalty functions satisfying the assumptions above are
the logarithmic penalty,
\begin{equation}
	\label{eq:log}
		\reg\slog(x; a) = \frac{1}{a} \log(1 + a \abs{x}),
\end{equation}
the arctangent penalty \cite{Selesnick_2013_arXiv_MSC},
\begin{equation}
	\label{eq:atan}
		\reg\satan(x; a) =
	 \frac{2}{a\sqrt{3}}  \, \left( 
	 	\tan\inv\left( \frac{1+2 a \abs{x}}{\sqrt{3}}\right) 	- \frac{\pi}{6}
		\right),
\end{equation}
and the first order rational function \cite{Geman_1992_PAMI}
\begin{equation}
	\label{eq:rat}
		\reg\srat(x; a) = \frac{\abs{x}}{1 + a \abs{x} / 2}.
\end{equation}
The rational penalty is defined for $ a \geq 0 $.
The log and atan penalties are defined for $ a > 0 $.
Note that as $ a \to 0 $, the three penalty functions approach the absolute value function.
They are illustrated in Fig.~\ref{fig:penalties}.

For later use, we record the value of the right-sided second derivative
of the three penalty functions:
\begin{equation}
	\label{eq:zplus}
	\reg''\slog(0^+; a) = \reg''\satan(0^+; a) = \reg''\srat(0^+; a) = -a.
\end{equation}

Note that the $\ell_p$ pseudo-norm ($ 0 < p < 1$),
i.e., $\reg(x) = \abs{x}^p$, does not 
satisfy the above assumptions.
It does not have unit slope at zero
nor can it be normalized or scaled to do so.

\begin{figure}
	\centering
		\includegraphics{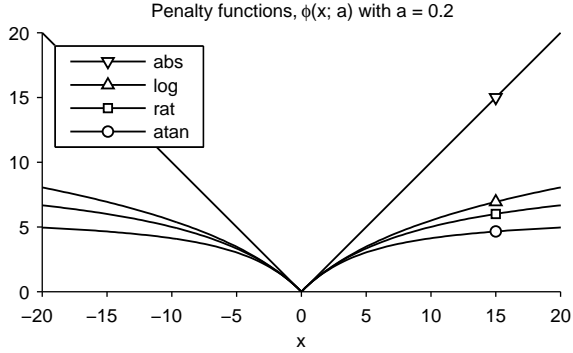}
	\caption{
	Several sparsity promoting penalty functions
	satisfying the assumptions in Sec.~\ref{sec:pen}.
	}
	\label{fig:penalties}
\end{figure}

\subsection{Threshold Functions}
\label{sec:thresh}

Proximity operators are a fundamental tool in efficient
sparse signal processing \cite{Combettes_2011_chap, Combettes_2008_SIAM}.
In the scalar case, proximity operators are thresholding/shrinkage function derived using a convex penalty function.
In this work, we utilize non-convex penalty functions; however, we can still define a threshold function
similar to the definition of a proximity operator.
The following proposition is closely related to 
Lemma 3.1 in \cite{Moulin_1999_Tinfo}
and 
Theorem 3.3 in \cite{Nikolova_2005_MMS},
both of which analyze the behavior of $ \tf $
for non-smooth,
not necessarily convex, $ \reg $.

\begin{prop}
\label{prop:scalar_thresh}
Define $ \tf : \RR \to \RR $ by
\begin{equation}
	\label{eq:deftf}
	\tf(y) =
	\arg \min_{x \in \RR} 	\;
	\Big\{ G(x) = \half \abs{y - x}^2 + \lam \reg(x) \Big\}
\end{equation}
where
$ G : \RR \to \RR $,
$ \lam > 0 $,
and 
$ \reg $ satisfies the assumptions in Sec.~\ref{sec:pen}.
Suppose also that $ G $ is strictly convex.
If $ \abs{y} \leq \lam $, then the unique minimizer of $ G $ is zero.
That is, $ \tf $ is a \emph{threshold} function and  $ \lam $ is the \emph{threshold} value.
\end{prop}
\begin{proof}
This is a special case of Proposition~\ref{prop:group_thres}
wherein $ \tf $ is a multivariate threshold function, $ \tf : \RR^K \to \RR^K $.
\end{proof}

Figure~\ref{fig:thresh_funs} illustrates threshold functions 
corresponding to several penalty functions.
We use $ \lam = 4 $ and $ a = 0.2 $.
The threshold function corresponding to the absolute value penalty function
is called the \emph{soft} threshold function \cite{Donoho_1994_Biom}.
Notice that, except for the soft threshold function,
the threshold functions approach the identity function.
The atan threshold function approaches identity the fastest.

The fact that the soft threshold function reduces large values by a constant amount is considered its deficiency. 
In the estimation of sparse signals in AWGN, this behavior results in a systematic underestimation (bias) of large
magnitude signal values \cite{Fan_2001_JASA}.
Hence, threshold functions that are asymptotically unbiased are often preferred to the soft threshold function, 
and the penalty functions from which they are derived promote sparsity more strongly than the $\ell_1$ norm \cite{Candes_2008_JFAP, Chartrand_2007_SPL, Gasso_2009_TSP, Foucart_2009_ACHA}.
The atan penalty function is derived specifically for its favorable behavior in this regard \cite{Selesnick_2013_arXiv_MSC}.

As shown in Ref.~\cite{Selesnick_2013_arXiv_MSC}, 
if $ \reg $ satisfies the above assumptions, then the right-sided derivative 
of $ \tf $ at the threshold 
is given by
\begin{equation}
	\tf'(\lam^+) = \frac{1}{1 + \lam \reg''(0^+)}.
\end{equation}
Hence, with parameters $ \lam = 4 $ and $ a = 0.2 $, we use \eqref{eq:zplus} to find that $ \tf'(\lam^+) = 5$
for $  \reg\slog $,  $ \reg\satan $, and $ \reg\srat $.
That is, each of these threshold functions in Fig.~\ref{fig:thresh_funs}
have the same derivative at $\lam^+$, but they approach the  identity at different rates.

\begin{figure}
	\centering
	\includegraphics{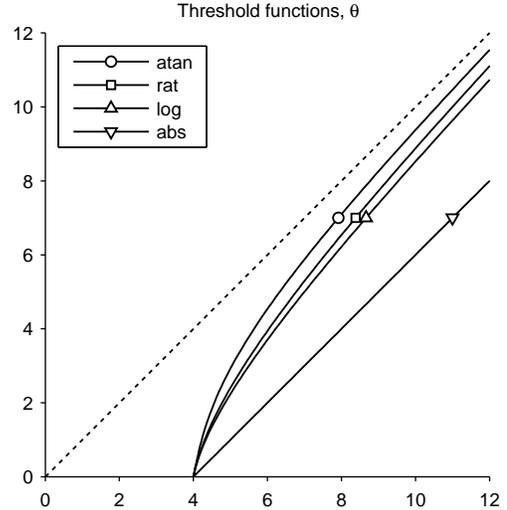}
	\caption{
	Threshold functions derived from the four penalty functions
	given in Sec.~\ref{sec:pen}; three of which are non-convex.
	}
	\label{fig:thresh_funs}
\end{figure}

\section{OGS with non-convex regularization}

For denoising group-sparse signals in AWGN,
we propose to minimize the cost function, $ F : \RR^N \to \RR $,
\begin{equation}
	\label{eq:defF}
	F(\x) = \half \norm{ \y - \x}_2^2  + \lambda \sum_{i \in \ZZ} \reg(\norm{\x_{i,K}}_2; a)  
\end{equation}
where $ \reg $ is a (non-convex) sparsity promoting penalty function satisfying the assumptions in Sec.~\ref{sec:pen},
and $ \lam > 0 $.
The group size,
$ K $ (a positive integer),
should be selected based on the size of the groups (clusters) arising in the data. 
This constitutes one's `prior knowledge' regarding the group sparsity behavior of the data 
and may need to be obtained through some trial-and-error.

In order to leverage convex optimization principles
and avoid non-convex optimization issues (local minima, sensitivity to noise, etc.),
we seek to restrict $ a $ so that $ F $ is strictly convex.
We note that the minimization of $ F $ is not so straight forward.
First, all the variables are coupled due to the overlapping group structure
of the regularization term.
That is, each component $ x(i) $ depends on every data sample $ y(k) $
(albeit the influence diminishes with distance $\abs{i-k}$).
Secondly, $ F $ is not differentiable.
In particular, $ F $ is generally not differentiable at the minimizer, $ \x\opt $,
due to the sparsity of $\x\opt$ induced by the regularizer.
(The penalty function, $\reg$, is non-differentiable at zero).
For these reasons, it is desirable that $ F $ be strictly convex.

In the following, we address the questions:
\begin{enumerate}
\item
For what values of $ a $ is $ F $ strictly convex?
\item
When $ F $ is strictly convex, how can the unique minimizer, $ \x\opt $, be efficiently computed?
\end{enumerate}

First, we make a few remarks.
If $ K = 1 $, then $ F $ simplifies to
\begin{equation}
	F(\x) = \sum_i \Bigl[ \half \abs{ y(i) - x(i)}^2  + \lambda \reg(x(i); a)  \Bigr],
\end{equation}
the components $ x(i) $ are not coupled, and the minimization of $ F $ 
amounts to component-wise non-linear thresholding; i.e., $ x\opt(i) = \tf( y(i); a ) $.
In this case, the cost function $ F $ does not promote any group structure.

If $ \reg $ is the absolute value function, i.e., $ \reg(x) = \abs{x}$,
then the cost function $ F $ in \eqref{eq:defF} is the same cost function
considered in our earlier work \cite{Chen_2014_OGS},
which considers only convex regularization.

If $ K = 1 $ and $ \reg $ is the absolute value function,
then the minimizer of $ F $ is given by point-wise soft thresholding of $ \y $.

The current work addresses the case $ K > 1 $ and $ \reg $ a non-convex regularizer,
so as to promote group sparsity more strongly in comparison to convex regularization.
The enhanced sparsity will be illustrated in Example 1 in Sect.~\ref{sec:Ex1}.

We also have the following result, similar to lemma 1 of Ref.~\cite{Yuan_2011_NIPS} which
considered convex regularizers promoting group sparsity.

\begin{lemma}
\label{lemma:minprop}

Let
$ \reg(\cdot, a) : \RR \to \RR $ satisfy the assumptions in Sec.~\ref{sec:pen}
and 
define $ F $ as in \eqref{eq:defF}.
Suppose $ F $ is strictly convex and that $ \x\opt $ is the minimizer of $ F $.
\begin{enumerate}
\item
If $ y(i) = 0 $ for some $ i $, then $ x\opt(i) = 0 $.
\item
If $ y(i) > 0 $ for some $ i $, then $ x\opt(i) \geq 0 $.
\item
If $ y(i) < 0 $ for some $ i $, then $ x\opt(i) \leq 0 $.
\item
$ \abs{ x\opt(i) } \leq \abs{y(i)},  \ \forall i \in \ZZ_N $.
\end{enumerate}
\end{lemma}
\begin{proof}
1) 
Define $S = \{i \in \ZZ_N : y(i)\neq 0\}$ and $\bar{S} = \ZZ_N \!\setminus\! S$. 
Given $ \x \in \RR^N $, 
define $\tilde{\x}\in\RR^N$ as $\tilde{x}(i) = x(i)$ for $ i \in S $, and $\tilde{x}(i) = 0$ for $i\in\bar{S}$.
For each group $i \in \ZZ_N$, we have $\norm{\x_{i,K}}_2 \geq \norm{\tilde{\x}_{i,K}}_2$.
Since $\reg(t)$ is increasing for $t \geq 0$, we have $\reg(\norm{\x_{i,K}}_2) \geq \reg(\norm{\tilde{\x}_{i,K}}_2)$. 
Therefore, for all $\x \in \RR^N$,
\begin{align*}
	F(\x) &=  \frac{1}{2} \norm{ \y - \x }_2^2  + \sum_i\lambda\reg(\norm{\x_{i,K}}_2; a) \label{eq:cost}
	\\
	      &= \frac{1}{2} \norm{ \y - \tilde{\x} }_2^2  + \half \sum_{i\in\bar{S}}|x(i)|^2  + \sum_i\lambda\reg(\norm{\x_{i,K}}_2; a)
	\\
	      &\geq \frac{1}{2} \norm{ \y - \tilde{\x} }_2^2    + \sum_i\lambda\reg(\norm{\tilde{\x}_{i,K}}_2; a)  
	\\    &=F(\tilde{\x}).
\end{align*}
This implies $x^*(i) = 0$ for $i\in\bar{S}$.

2)
Proof by contradiction.
Suppose $y(i) > 0$, but $x\opt(i) < 0$ for some $ i $.
Define $\tilde{\x}$ by $\tilde{x}(i) = 0$, and $\tilde{x}(n) = x\opt(n)$ for $n \neq i$.
It can be shown as in 1) that $F(\x\opt) > F(\tilde{\x})$.
This contradicts the optimality of $\x\opt$.

3) 
The proof is like 2).

4)
Proof by contradiction.
Suppose $ y(i) \geq 0 $, but $x\opt(i) > y(i)$
for some $ i $.
Define $\tilde{\x}$ by $\tilde{x}(i) = y(i) $, and $\tilde{x}(n) = x\opt(n)$ for $n \neq i$.
It can shown as in 1) that $F(\x\opt) > F(\tilde{\x})$.
This contradicts the optimality of $\x\opt$.
Together with 2), it follows that if $ y(i) \geq 0 $, then  $ 0 \leq x\opt(i) \leq y(i) $.
Similarly, 
 if $ y(i) \leq 0 $, then  $ y(i) \leq x\opt(i) \leq 0 $.
\end{proof}

\subsection{Group Thresholding}

In order to determine the convexity of $ F $, we first consider a simpler cost function, $ H $,
which consists of a single group.
What values of $ a $ ensure that $ H $ is strictly convex?

\begin{prop}
\label{prop:cvxcond}
Consider the functions  $ H : \RR^K \to \RR $ and $G : \RR \to \RR$,
defined as
\begin{align}
	H(\x) &= \half \norm{ \bar{\y} - \x}_2^2  + \lambda \reg(\norm{\x}_2; a)
	\\
	G(v)  &= \half |\bar{y} - v|^2  + \lambda \reg(v; a)
\end{align}
where
$ \bar{\y} \in \RR^K $, 
$ \bar{y} \in \RR$,
$ \lam > 0 $,
and
$ \reg(\cdot, a) : \RR \to \RR $ satisfies the assumptions in Sec.~\ref{sec:pen}.
Then $ H $ is strictly convex iff $G$ is strictly convex.
Furthermore, if
\begin{equation}
	\label{eq:regdplam}
	\reg''(0^+; a) > -\frac{1}{\lam},
\end{equation}
then $H$ and $G$ are both strictly convex.
\end{prop}
\begin{proof}
Let us expand $ H$ and $ G $ as
\begin{align}
	H(\x) &= \half \norm{ \bar{\y} }_2^2 + \half \norm{\x}_2^2 - \bar{\y}\tr \x + \lambda \reg(\norm{\x}_2; a)
	\\
	G(v) &= \half |\bar{y}|^2 + \half |v|^2 - \bar{y}x + \lambda \reg(v; a).
\end{align}
Define
$ A : \RR^K \to \RR $ and $B : \RR \to \RR$,
\begin{align}
	A(\x) &= 0.5 \norm{\x}_2^2 + \lambda \reg(\norm{\x}_2; a),
	\\
	B(v)  &= 0.5 \abs{v}^2 +  \lambda \reg(v; a).
	\label{eq:B}
\end{align}
It can be observed that $A$ is strictly convex iff $H$ is strictly convex.
Similarly $B$ is strictly convex iff $G$ is.

We claim that $ A $  is strictly convex 
if and only if the function $ B $ is strictly convex. 
[Note that $ A(\x) = B(\norm{\x}_2)$.]

Suppose $B$ is  strictly convex. From \eqref{eq:B}, $B$ is increasing on $\nnrl$.
Based on the convexity of $\norm{\x}_2$ and Proposition~2.1.7 of 
Ref.~\cite[page~89]{Hiriart_2001}, $B(\norm{\x}_2)$ is strictly convex,
and hence $A$ is strictly convex.

Suppose $A$ is strictly convex. Given $v_1,v_2\in\RR$, define $\x_1 = v_1\,\e$ and $\x_2 = v_2\,\e$ with $\norm{\e}_2 = 1$.
Note that $B$ in \eqref{eq:B} is symmetric.
For all $\alpha, \beta$ satisfying $\alpha\in (0, 1)$ and $\alpha + \beta = 1$, we have
\begin{align*}
	B(\alpha v_1 + \beta v_2) & = B(\abs{\alpha v_1 + \beta v_2})
	= A(\alpha \x_1 + \beta \x_2)
	\\ 
	& <  \alpha A(\x_1) + \beta A( \x_2) 
 	= \alpha B(\abs{v_1}) + \beta B( \abs{v_2})
 	\\
	& = \alpha B(v_1) + \beta B( v_2),
\end{align*}
which implies the strict convexity of $B$. 

To prove the second part of this proposition,
according to the assumptions on $\reg$ in Sec.~\ref{sec:pen},
$ B $ is continuous on $\RR$,
twice differentiable on $ \RNZ $,
and symmetric.
Hence, from Corollary~\ref{thm:convexity} (see Appendix), it is sufficient that $B''$ be positive on $\RNZ$
and that $ B'(0^-) < B'(0^+) $.

Note that $ B'(0^+) = \lam \reg(0^+; a)  = 1 $; 
and by symmetry $ B'(0^-) = \lam \reg(0^-; a)  = -1 $.
Hence  $ B'(0^-) < B'(0^+) $.

To ensure the second derivative of $ B $ is positive on $\RNZ$,
we have the condition 
\begin{equation}
	B''(v) = 1 + \lambda \reg''(v; a) > 0, \quad \text{for $v > 0$}
\end{equation}
or
\begin{equation}
	\reg''(v; a) > -\frac{1}{\lam}, \quad v > 0.
\end{equation}
Due to assumption 7 on $\reg$
[i.e., $ \reg''(0^+) \leq \reg''(x), \; \forall x > 0 $],
we have \eqref{eq:regdplam}.
\end{proof}

The condition \eqref{eq:regdplam} can be used to determine values of $ a $ that ensure strict convexity of $ H $.
For the log, atan, and rational penalty functions
($  \reg\slog $,  $ \reg\satan $, $ \reg\srat $),
we use \eqref{eq:zplus} to obtain the following intervals for $ a $ ensuring strict convexity of $H$.

\begin{cor}
Suppose $ \reg $ is one of the penalty functions given in Sec.~\ref{sec:pen} 
($  \reg\slog $,  $ \reg\satan $, $ \reg\srat $).
Then $ H $ is strictly convex if
\begin{equation}
	0 < a < \frac{1}{\lam}.
\end{equation}
\end{cor}

Based on a strictly convex function $ H $, 
we may define a multivariate threshold/shrinkage function $ \tf : \RR^K \to \RR^K $
as in the scalar case \eqref{eq:deftf}.
It is informative to note the threshold of the multivariate thresholding function.

\begin{prop}
\label{prop:group_thres}
Define $ \tf : \RR^K \to \RR^K $ by
\begin{equation}
	\label{eq:mdeftf}
	\tf(\y) =
	\arg \min_{\x \in \RR^K} 	\;
	\Big\{ H(\x) = \half \norm{\y - \x}_2^2 + \lam \reg( \norm{\x}_2) \Big\}
\end{equation}
where
$ \lam > 0 $,
and 
$ \reg $ satisfies the assumptions in Sec.~\ref{sec:pen}.
Suppose also that $ H $ is strictly convex.
If $ \norm{\y}_2 < \lam $, then the unique minimizer of $ H $ is the zero vector.
That is, $ \tf $ is a multivariate threshold function with 
threshold  $ \lam $.
\end{prop}
\begin{proof}
We consider the subgradient  of a convex function $H$ (see Ref.~\cite[Definition~1.1.4, page~165]{Hiriart_2001}), denoted by $\partial H(\x)$, is equal to $\x - \y + \partial \reg(\norm{\x}_2)$.
Since $\reg'(0^+) = 1$ by assumption~6 in Sec.~\ref{sec:pen}, we have $\partial \reg(\norm{\0}_2) = \partial \{ \|\0\|_2 \}$, which is equal to $\{\bv\in\RR^K,\; \norm{\bv}_2 \leq 1\}$.

This leads to
\begin{equation}
	\partial H(\0) = \{\lam\,\bv - \y:\; \norm{\bv}_2 \leq 1 \}.
\end{equation}
Since $\x^*$ is a minimizer of $H$ iff $\0 \in \partial H(\x^*)$ (see \cite[Theorem~2.2.1, page 177]{Hiriart_2001}), we deduce the following.
\begin{itemize}
	\item
	Suppose $\norm{\y}_2\leq \lam$.
	We can choose $\bv = \y/\lam$ satisfying $\norm{\bv}_2 \leq 1$ such that $\lam\,\bv - \y = \0$.
	We have $\0\in \partial H(\0)$, which implies that $\0$ is the minimizer of $H$. 
	\item
	Suppose $\norm{\y}_2 > \lam$.
	There is no $\bv$ satisfying $\norm{\bv}_2 \leq 1$ such that $\lam\,\bv - \y  = \0$.
	Hence, $\0$ is not the minimizer of $H$.
\end{itemize}
From the arguments above, we conclude that $\lam$ defines the threshold of $\theta$.
\end{proof}

When $ \reg $ is the absolute value function, 
the induced multivariate threshold function $ \tf $
can be expressed in closed form \cite{SenSel-02-tsp}.
(Essentially, it performs soft-thresholding on the 2-norm.)
A generalization to the case where the data consistency
term in \eqref{eq:mdeftf} is of the form $\norm{\y-\A\x}_2^2$ has also been
addressed \cite{Tibau_2011_SPL}.
We note that neither \cite{Tibau_2011_SPL} nor \cite{SenSel-02-tsp}
consider either non-convex regularization or overlapping group sparsity.

If the penalty function, $ \reg $, is strictly concave on the positive real line (log, atan, etc.),
then the induced multivariate threshold function results in less bias
of large magnitude components; i.e., $ \tf(\y) $ approaches the identity
function for large $\y$.
An exploration along these lines is given in \cite{Selesnick_2008_TSP_MLap};
however, in that work, the non-convexity was quite mild and not adjustable.
(The non-convex regularization in \cite{Selesnick_2008_TSP_MLap} is based
on the multivariate Laplace probability density function, which does not have a 
shape parameter, analogous to $ a $ in the current work.)
Furthermore, overlapping group sparsity is not considered in \cite{Selesnick_2008_TSP_MLap}.

\subsection{Overlapping Group Thresholding}

Using the results above, 
we can find a condition on $ a $ to ensure $ F $ in \eqref{eq:defF} is strictly convex.
The result permits the use of non-convex regularization 
to strongly promote group sparsity while preserving strict convexity of the total cost function, $ F $.

\begin{thm}
\label{prop:OGcvxcond}
Consider  $ F : \RR^N \to \RR $, defined as
\begin{equation}
	\label{eq:defFB}
	F(\x) = \half \norm{ \y - \x}_2^2  + \lambda \sum_i \reg(\norm{\x_{i,K}}_2; a)  
\end{equation}
where
$\y \in \RR^N $, 
$ K \in \ZZ_+$,
$ \lam > 0 $,
and
$ \reg(\cdot, a) : \RR \to \RR $ satisfies the assumptions in Sec.~\ref{sec:pen}.
Then $ F $ is strictly convex if
\begin{equation}
	\label{eq:OGregdplam}
	\reg''(0^+; a) > -\frac{1}{K\lam}.
\end{equation}
\end{thm}
\begin{proof}
Write $ F $ as
\begin{equation}
	F(\x) = \sum_i F_i(\x_{i,K})
\end{equation}
where $ F_i : \RR^K \to \RR $ is defined as
\begin{equation}
	\label{eq:defFi}
	F_i(\bv)
	= \frac{1}{2K} \norm{ \y_{i,K} - \bv}_2^2  + \lambda \reg(\norm{\bv}_2; a)
\end{equation}
for $ i \in \ZZ $.
Suppose \eqref{eq:OGregdplam} is satisfied.
Then by Prop.~\ref{prop:cvxcond} the functions $ F_i $ are strictly convex.
Since $F$ is a sum of strictly convex functions, $ F $ is strictly convex.
\end{proof}

\begin{cor}
\label{cor:OG}
Suppose $ \reg $ is one of the penalty functions given in Sec.~\ref{sec:pen} 
($  \reg\slog $,  $ \reg\satan $, $ \reg\srat $).
Then $ F $ is strictly convex if
\begin{equation}
	\label{eq:OG}	
	0 < a < \frac{1}{K\lam}.
\end{equation}
\end{cor}

We give some practical comments on using \eqref{eq:OG} to set the parameters $\{K, \lam, a \}$.
We suggest that $ K $ be chosen first, based on the structural properties of the signal to be denoised.
We suggest that $ a $ then be set to a fixed fraction of its maximal value; i.e., fix $ \beta \in [0, 1]$ and set $ a = \beta/(K \lam) $.
So, we consider $ a $ as a function of $ \lam $.
We then set $ \lam $ according to the noise variance.
In Sec.~\ref{sec:setlam}, we describe two approaches for the selection of $ \lam $.
In our numerical experiments on speech enhancement, we have found that setting $ a $ to its maximal 
value
of $ 1/(K \lam) $
generally yields the best results; i.e., $ \beta = 1 $.
Hence, in the examples in Sec.~\ref{sec:examples}, we set $ a $ to its maximal value.

Equation \eqref{eq:OG} may suggest the proposed method becomes ineffective for large $ K$.
It can be noted from \eqref{eq:OG} that
for large $ K $, $ a \lam $ should be small ($ < 1/K$).
If $ \lam $ is set so as to achieve a desired degree of noise suppression,
then \eqref{eq:OG} implies $ a $ should be small.
A small $ a $, in turn, limits the non-convexity of the regularizer.
Hence, 
it appears the benefit of the proposed non-convex regularization method is diminished for large $ K $.
However, two considerations offset this reasoning. 
First, for larger $ K $, a smaller value of $ \lam $ is needed so as to achieve a fixed level of noise suppression
(this can be seen, for example, in Table III).
Secondly, for larger $K$, there is greater overlap between adjacent groups
because the groups are fully-overlapping;
so, regularization may be more sensitive to $ a $.

\subsection{Minimization Algorithm}

To derive an algorithm minimizing the strictly convex function $ F $ in \eqref{eq:defF},
we use the majorization-minimization (MM) procedure \cite{FBDN_2007_TIP}
as in \cite{Chen_2014_OGS}.
The MM procedure replaces a single minimization problem by a sequence of (simpler) ones.
Specifically, MM is based on the iteration
\begin{equation}
	\label{eq:MM}
	\x\iter{k+1} = \arg \min_{\x} Q(\x, \x\iter{k})
\end{equation}
where the function, $ Q : \RR^N \times \RR^N \to \RR $, is a majorizer (upper bound) of $ F $
and
$ k $ is the iteration index.
For $ Q $ to be a majorizer of $ F $ 
it should satisfy
\begin{align}
	Q(\x, \bv) & \geq F(\x), \; \forall \x \in \RR^N
	\\
	Q(\bv, \bv) & = F(\bv). 
\end{align}
The MM procedure monotonically reduces the cost function value at each iteration.
Under mild conditions, the sequence $ \x\iter{k} $ converges to the minimizer of $ F $  \cite{FBDN_2007_TIP}.

To specify a majorizer of the cost function $ F $ in \eqref{eq:defF}, we first specify a majorizer 
of the penalty function, $ \reg $.
To simplify notation, we suppress the dependence of $ \reg $ on $ a $.

\begin{figure}
	\centering
	\includegraphics{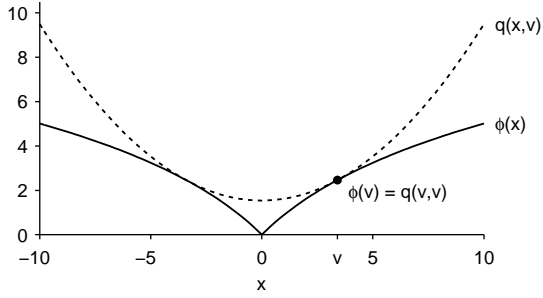}
	\caption{
		Majorization of non-convex $ \reg(x) $ by $ q(x,v) $.
	}
	\label{fig:major}
\end{figure}

\begin{lemma}
\label{lemma:major}
Assume $ \reg : \RR \to \RR $ satisfies the assumptions in Sec.~\ref{sec:pen}.
Then $ q : \RR \times \RR \to \RR $, defined by
\begin{equation}
	\label{eq:qmaj}
	q(x, v) = \frac{1}{2 v} \reg'(v) \, x^2 + \reg(v) - \frac{v}{2} \reg'(v),
\end{equation}
is a majorizer of $ \reg $ except for $ v = 0 $,
i.e.,
\begin{align}
	\label{eq:majA}
	q(x, v) & \geq \reg(x), \ \forall x \in \RR, \ \forall v \in \RNZ
	\\
	\label{eq:majB}
	q(v, v) & = \reg(v),  \ \forall v \in \RNZ
\end{align}
\end{lemma}
The majorization of $ \reg(x) $ by $ q(x, v) $ is illustrated in Fig.~\ref{fig:major}.
\begin{proof}
By direct substitution, one may verify \eqref{eq:majB}.
We now show \eqref{eq:majA}.
Let $v > 0$ and $x \geq 0$.
Using Taylor's theorem \cite[Theorem~5.15]{Rudin_1976},
we have
\begin{equation}
	\label{eq:majproof1}
	\reg(x) = \reg(v) + \reg'(v)(x-v) + \frac{\reg''(v_0)}{2}(x-v)^2
\end{equation}
for some $v_0$ between $x$ and $v$.
By the assumptions on $ \reg $, we have  $\reg''(v_0) < 0$.
Hence from \eqref{eq:majproof1},
\begin{equation}
	\label{eq:majproof2}
	\reg(x) \leq \reg(v) + \reg'(v)(x-v).
\end{equation}
Note that  $(x - v)^2\geq 0$ implies
\begin{equation}
	\label{eq:majproof3}
	x \leq \frac{1}{2v}x^2 + \frac{v}{2}.
\end{equation}
Using \eqref{eq:majproof3} for $ x $ on the left-hand side of \eqref{eq:majproof2} gives
\begin{equation}
	\label{eq:majproof4}
	\reg(x) \leq \reg(v) + \reg'(v) \Bigl( \frac{1}{2v}x^2 + \frac{v}{2} - v \Bigr).
\end{equation}
Recognizing that the right-hand side of \eqref{eq:majproof4} is $ q(x, v) $, 
we obtain $ \reg(x) \leq q(x, v) $ 
for all $ x \geq 0 $, $ v > 0 $.
By symmetry of $ q $ and $ \reg $, 
we obtain \eqref{eq:majA}.
\end{proof}

Since $ q $ is a majorizer of $ \reg $,  the function
\begin{equation}
	Q(\x, \bv) = \half \norm{ \y - \x}_2^2  + \lambda \sum_i q(\norm{\x_{i,K}}_2, \norm{\bv_{i,K}}_2)
\end{equation}
is a majorizer of $ F $.
Using \eqref{eq:qmaj},  the function $ Q $ is given by
\[
	Q(\x, \bv) = \half \norm{ \y - \x}_2^2   + \frac{\lam}{2} \sum_i  \frac{\reg'(\norm{\bv_{i,K}}_2)}{\norm{\bv_{i,K}}_2}   \norm{\x_{i,K}}_2^2 + C
\]
where $ C $ does not depend on $ \x $.
After algebraic manipulations, $ Q $ can be expressed as
\begin{equation}
	\label{eq:Qfun}
	Q(\x, \bv) = \half \norm{ \y - \x}_2^2   + \frac{\lam}{2} \sum_i  r(i; \bv) \,  x^2(i) + C
\end{equation}
where $ r : \ZZ \times \RR^K \to \RR $ is defined as
\begin{equation}
	\label{eq:riterm}
	r(i; \bv) = \sum_{j=0}^{K-1}  \frac{\reg'(\norm{\bv_{i-j, K}}_2)}{ \norm{\bv_{i-j, K}}_2 }.
\end{equation}
Note that the components $ x(i) $ in  \eqref{eq:Qfun} are uncoupled.
Furthermore, $ Q $ is quadratic in $ x(i) $.
Hence, the minimizer of $ Q $ with respect to $ \x $ is easily obtained.
The quantities $ r(i, \bv) $ in \eqref{eq:riterm} are readily computed;
$r$ is essentially a double $K$-point convolution, with a nonlinearity 
between the two convolutions.

\begin{table}
	\caption{
		\label{table:ogs}
		Overlapping group shrinkage (OGS) with penalty $\reg$.
	}
\begin{center}
\begin{minipage}{2.7in}
\hrule
\begin{align*}
	&	\text{input: $\y \in \RR^N $, $ \lambda > 0 $, $ K $, $ \reg $}
	\\
	&
		\x = \y	\qquad \quad \text{(initialization)}
	\\
	&
		\NZ =  \{ i \in \ZZ_N :   y(i) \neq 0 \}  
	\\
	& \text{repeat}
	\\
	&	\qquad
			a(i) = \left[	\sum_{k=0}^{K-1} \abs{x(i+j)}^2 \right]^{1/2} \!,
			\ i \in \NZ
	\\[1em]
	&	\qquad
			b(i) = \frac{\reg'(a(i))}{a(i)}, \quad i \in \NZ
	\\[1em]
	&	\qquad
			r(i) = \sum_{j = 0}^{K-1}  b(i-j), \quad \ i \in \NZ
	\\[1em]
	&	\qquad
			x(i) = \frac{y(i)}{1 + \lam \, r(i)}, \quad i \in \NZ 
	\\[1em]
	&	\qquad
		\NZ =  \{ i \in \ZZ_N : \abs{x(i)}  >  \eps \}  \hspace{2em} (*)
	\\
	& \text{until convergence}
	\\
	& \text{return: $\x$}
\end{align*}
\hrule
$(*)$ For finite precision implementations only.
\end{minipage}
\end{center}
\end{table}

Using \eqref{eq:Qfun} in the MM iteration \eqref{eq:MM}, we obtain 
\begin{equation}
	\label{eq:ogsit}
	x\iter{k+1}(i) = \frac{y(i)}{1 + \lam \, r(i; \x\iter{k}) },
	\quad
	i \in \ZZ_N,
\end{equation}
where $ r $ is given by \eqref{eq:riterm}.
This constitutes the OGS algorithm. 
The algorithm is summarized in Table~\ref{table:ogs}.
We denote the output of the OGS algorithm 
as $ \y = \ogs(\x; \lam, K, \reg) $.

Note that $ q $ in \eqref{eq:qmaj} is undefined if $ v = 0 $.
This singularity issue often arises when a quadratic function is used to 
majorize a non-smooth function \cite{Oliveira_2009_SP, FBDN_2007_TIP}.
This issue may manifest itself in the OGS algorithm whenever a $K$-point
group of $ \x $ is equal to the $K$-point zero vector;
i.e., if $ \x_{i,K}\iter{k} = \0 \in \RR^K $ for some index $ i $ and iteration $k$.
In the event of such an occurrence, the OGS algorithm would encounter a `divide-by-zero' error.
However, such an occurrence is guaranteed not to occur with suitable initialization,
as described in \cite{Chen_2014_OGS}.
For example, it is sufficient to initialize all $x(i)$ to non-zero values, i.e., $x\iter{0}(i) \neq 0 $ for all $ i \in \ZZ_N $.
With such an initialization, it is readily observed that 
$ r(i; \x\iter{k}) $ in the denominator of \eqref{eq:ogsit} is strictly positive and finite
and that $ x\iter{k}(i) \neq 0 $
for all $ i \in \ZZ_N $ and all iterations $k$.
When some components of the solution $\x\opt$ are zero (as expected, due to sparse regularization),
those values $x\iter{k}(i)$ approach zero in limit; 
i.e., $ x\iter{k}(i) \to 0 $ as $ k \to \infty $.

We propose initializing $ \x $ to $ \y $; i.e.,  $ \x\iter{0} = \y $,
and we exclude from the iteration \eqref{eq:ogsit} those $ i $ for which $ y(i) = 0 $.
The set $ S \subset \ZZ_N $ in Table~\ref{table:ogs}
serves to exclude these components from the iterative update.
In this case, $ x\iter{k}(i) = 0 $ for all iterations $k$, which is justified by part 1) of lemma \ref{lemma:minprop}.
As a consequence of lemma \ref{lemma:minprop},
initializing $ x\iter{0}(i) $ to zero for $ i \notin S $ is optimal.
Therefore, the algorithm excludes these values from the update procedure because they are already optimal.

With the initialization $ \x\iter{0} = \y $, it is readily observed, as above, that 
$ r(i; \x\iter{k}) $ in the denominator of \eqref{eq:ogsit} is strictly positive and finite
and that $ x\iter{k}(i) \neq 0 $
for all $ i \in S $ and all iterations $k$.
Assuming infinite precision, it is sufficient to define $ S $ prior to the loop only;
the last line in Table~\ref{table:ogs}, indicated by $(*)$, can be omitted.
It is guaranteed that a division by zero will never occur, as discussed above.

The OGS algorithm proceeds by gradually attenuating the $x(i)$, $ i \in S $, toward their optimal values (including zero).
The attenuation is multiplicative, so the the value never equals zero, even though it may converge to zero.
But if many values reach `machine epsilon'  
then a divide by zero may subsequently occur in the implementation. 
Hence, to avoid possible divide-by-zero errors due to finite precision arithmetic,
the OGS algorithm updates $ S $ at the end of the loop in Table~\ref{table:ogs}.
The small number, $ \eps $, may be set to `machine epsilon', which for single precision floating point is about $ 10^{-16} $.
This value is usually considered the same as zero.

We do not prove the convergence of the OGS algorithm to the minimizer of $ F $
due to the complication of the singularity issue.
However, due to its derivation based on the majorization-minimization principle,
OGS is guaranteed to decrease the cost function at each iteration.
Moreover, in practice, we have observed through extensive numerical investigation, 
that the algorithm has the same rapid convergence behavior
as convex regularized OGS \cite{Chen_2014_OGS}.

Note that in the OGS algorithm, summarized in Table~\ref{table:ogs},  the penalty function appears in only one place:
the computation of $b(i)$.
It can therefore be observed that the role of the penalty is encapsulated by the function $ \reg'(u)/u $.
Table \ref{table:wts} lists this function
for the penalty functions given in Sec.~\ref{sec:pen}.
The function $ \reg'(u)/u $ have very similar functional forms.
The similarity of these functions reveal the close relationship among the listed penalty functions.

\begin{table}
\caption{Sparse penalties and corresponding nonlinearities
\label{table:wts}}
\centering
\begin{tabular}{@{}lll@{}}
	\toprule
	penalty	& $\displaystyle \reg(u)$ 		& $ \displaystyle \reg'(u)/u $ 
	\\
	\midrule
	abs	&
   	$\displaystyle \abs{u}$ \hspace{1em}		&	$ \displaystyle \frac{1}{\abs{u}}$
	\\[1.5em]
	log	&
	$ \displaystyle \frac{1}{a} \log(1 + a\abs{u})$
	&	$ \displaystyle \frac{1}{ \abs{u} (1 + a \abs{u}) } $
	\\[1.5em]
	atan	&
   $ \displaystyle	 \frac{2}{a\sqrt{3}}   \!  \left(  \!
	 	\tan\inv \! \left( \frac{1+2 a \abs{u}}{\sqrt{3}}\right)
		-
		\frac{\pi}{6} \!
	\right)
$
	&	$  \displaystyle  \frac{1}{  \abs{u} (1 + a \abs{u} + a^2 \abs{u}^2) } $
	\\[1.5em]
	rational	&
	$ \displaystyle \frac{\abs{u}}{1 + a \abs{u} / 2} $
	&	$ \displaystyle \frac{1}{  \abs{u} (1 + a \abs{u})^2 } $
   \\
   \bottomrule
\end{tabular}
\end{table}

\subsection{The Multidimensional Case}

The results and algorithm described in the preceding sections
can be extended to the multidimensional case straightforwardly.
In the numerical experiments below,
we use a two-dimensional version of the algorithm in order to
denoise the time-frequency spectrogram of a noisy speech waveform.

Suppose $ \x $ is a 2D array of size $ N_1 \times N_2 $; i.e., 
\[
	\x = \{ x(i_1, i_2), \  0 \leq i_1 \leq N_1-1, \  0 \leq i_2 \leq N_2-1 \}.
\]
The array can be expressed using multi-indices as
\[
	\x = \{ x(i), \ i \in \ZZ_{N_1} \! \times \ZZ_{N_2} \}.
\]
Let $ K = (K_1, K_2) $ denote the size of a 2D group.
Then a sub-group of size $ K $ can be expressed as
\[
	\x_{i, K} = \{ x(i+j), \ j \in \ZZ_{K_1} \! \times \ZZ_{K_2} \}.
\]
In the two-dimensional case, the function $ F $ in \eqref{eq:defF} is 
\begin{equation}
	\label{eq:defFB2D}
	F(\x) = \sum_{i \in \ZZ^2} \, \half \abs{y(i) - x(i)}^2
	+
	\lam \, \reg(\norm{\x_{i,  K}}_2; a),
\end{equation}
and  conditions 
\eqref{eq:OGregdplam}
and
\eqref{eq:OG}	
become 
\begin{equation}
	\label{eq:OGregdplam2}
	\reg''(0^+; a) > -\frac{1}{K_1 K_2 \lam}
\end{equation}
and 
\begin{equation}
	\label{eq:OG2}	
	0 < a < \frac{1}{K_1 K_2 \lam}
\end{equation}
respectively.
The algorithm in Table~\ref{table:ogs} is essentially the same
for the two-dimensional case. 
The summations become double summations, etc.
Extensions to higher dimensional signals is similarly straight forward.

\subsection{Regularization Parameter Selection}

\label{sec:setlam}

\noindent
\textbf{Noise level suppression.}
The regularization parameter, $ \lam $, can be selected using existing generic techniques such as the L-curve method.
However, in \cite{Chen_2014_OGS} we described an approach to set $ \lam $ based directly on the standard deviation, $\sigma$, of the AWGN, which we assume is known. This approach seeks to preserve one of the concepts of scalar thresholding (e.g., hard or soft thresholding), namely the processing of signal values based on relative magnitude.  Consider the problem of estimating a sparse signal in AWGN.
If many of the non-zero values of the sparse signal exceed the noise floor,
then a suitable threshold value, $T$, should exceed the noise floor.
But $ T $ should not be too large, or else the non-zero values of the sparse signal will be annihilated.
Hence, it is reasonable to use the value $ T = 3\sigma $.
This threshold will set most of noise (about 99.7\%) to zero.
(If the sparse signal has non-zero values less than $ T $ in magnitude, then those values will be lost.)

\begin{table*}
\caption{
	OGS regularization parameter with penalty $\reg(\cdot) = \reg\satan(\cdot,1/(K_1 K_2 \lam))$ and 25 iterations
	\label{table:lam}
}
\centering
\begin{tabular}{@{}lccccc@{}}
\toprule
$K$        &  \multicolumn{5}{c}{    $\lam$, $\alpha(\lam, K, \reg)$      } \\
\midrule
$ 1\times  1$ & $ 4.25,\,1.00\cdot 10^{ -2}$ & $ 4.59,\,4.33\cdot 10^{ -3}$ & $ 4.93,\,1.51\cdot 10^{ -3}$ & $ 5.27,\,4.05\cdot 10^{ -4}$ & $ 5.61,\,1.00\cdot 10^{ -4}$\\
$ 1\times  2$ & $ 2.14,\,1.00\cdot 10^{ -2}$ & $ 2.31,\,4.35\cdot 10^{ -3}$ & $ 2.48,\,1.49\cdot 10^{ -3}$ & $ 2.64,\,3.99\cdot 10^{ -4}$ & $ 2.81,\,1.00\cdot 10^{ -4}$\\
$ 1\times  3$ & $ 1.45,\,1.00\cdot 10^{ -2}$ & $ 1.56,\,4.52\cdot 10^{ -3}$ & $ 1.68,\,1.56\cdot 10^{ -3}$ & $ 1.79,\,4.06\cdot 10^{ -4}$ & $ 1.91,\,1.00\cdot 10^{ -4}$\\
$ 1\times  4$ & $ 1.11,\,1.00\cdot 10^{ -2}$ & $ 1.20,\,4.47\cdot 10^{ -3}$ & $ 1.29,\,1.58\cdot 10^{ -3}$ & $ 1.38,\,4.11\cdot 10^{ -4}$ & $ 1.47,\,1.00\cdot 10^{ -4}$\\
$ 1\times  5$ & $ 0.91,\,1.00\cdot 10^{ -2}$ & $ 0.98,\,4.37\cdot 10^{ -3}$ & $ 1.05,\,1.55\cdot 10^{ -3}$ & $ 1.13,\,4.07\cdot 10^{ -4}$ & $ 1.20,\,1.00\cdot 10^{ -4}$\\
$ 2\times  2$ & $ 1.08,\,1.00\cdot 10^{ -2}$ & $ 1.16,\,4.37\cdot 10^{ -3}$ & $ 1.24,\,1.47\cdot 10^{ -3}$ & $ 1.33,\,3.95\cdot 10^{ -4}$ & $ 1.41,\,1.00\cdot 10^{ -4}$\\
$ 2\times  3$ & $ 0.73,\,1.00\cdot 10^{ -2}$ & $ 0.79,\,4.41\cdot 10^{ -3}$ & $ 0.85,\,1.49\cdot 10^{ -3}$ & $ 0.90,\,3.96\cdot 10^{ -4}$ & $ 0.96,\,1.00\cdot 10^{ -4}$\\
$ 2\times  4$ & $ 0.56,\,1.00\cdot 10^{ -2}$ & $ 0.61,\,4.18\cdot 10^{ -3}$ & $ 0.65,\,1.44\cdot 10^{ -3}$ & $ 0.70,\,3.91\cdot 10^{ -4}$ & $ 0.74,\,1.00\cdot 10^{ -4}$\\
$ 2\times  5$ & $ 0.47,\,1.00\cdot 10^{ -2}$ & $ 0.50,\,3.89\cdot 10^{ -3}$ & $ 0.54,\,1.33\cdot 10^{ -3}$ & $ 0.58,\,3.74\cdot 10^{ -4}$ & $ 0.61,\,1.00\cdot 10^{ -4}$\\
$ 3\times  3$ & $ 0.50,\,1.00\cdot 10^{ -2}$ & $ 0.54,\,4.11\cdot 10^{ -3}$ & $ 0.58,\,1.38\cdot 10^{ -3}$ & $ 0.62,\,3.81\cdot 10^{ -4}$ & $ 0.66,\,1.00\cdot 10^{ -4}$\\
$ 3\times  4$ & $ 0.40,\,1.00\cdot 10^{ -2}$ & $ 0.43,\,3.57\cdot 10^{ -3}$ & $ 0.46,\,1.19\cdot 10^{ -3}$ & $ 0.49,\,3.51\cdot 10^{ -4}$ & $ 0.51,\,1.00\cdot 10^{ -4}$\\
$ 3\times  5$ & $ 0.34,\,1.00\cdot 10^{ -2}$ & $ 0.36,\,3.26\cdot 10^{ -3}$ & $ 0.39,\,1.04\cdot 10^{ -3}$ & $ 0.41,\,3.23\cdot 10^{ -4}$ & $ 0.43,\,1.00\cdot 10^{ -4}$\\
$ 4\times  4$ & $ 0.33,\,1.00\cdot 10^{ -2}$ & $ 0.35,\,3.24\cdot 10^{ -3}$ & $ 0.37,\,1.02\cdot 10^{ -3}$ & $ 0.39,\,3.16\cdot 10^{ -4}$ & $ 0.41,\,1.00\cdot 10^{ -4}$\\
$ 4\times  5$ & $ 0.29,\,1.00\cdot 10^{ -2}$ & $ 0.30,\,3.09\cdot 10^{ -3}$ & $ 0.32,\,9.61\cdot 10^{ -4}$ & $ 0.33,\,3.04\cdot 10^{ -4}$ & $ 0.35,\,1.00\cdot 10^{ -4}$\\
$ 5\times  5$ & $ 0.25,\,1.00\cdot 10^{ -2}$ & $ 0.26,\,3.05\cdot 10^{ -3}$ & $ 0.28,\,9.44\cdot 10^{ -4}$ & $ 0.29,\,3.01\cdot 10^{ -4}$ & $ 0.30,\,1.00\cdot 10^{ -4}$\\
$ 2\times  8$ & $ 0.33,\,1.00\cdot 10^{ -2}$ & $ 0.35,\,3.33\cdot 10^{ -3}$ & $ 0.37,\,1.05\cdot 10^{ -3}$ & $ 0.39,\,3.22\cdot 10^{ -4}$ & $ 0.41,\,1.00\cdot 10^{ -4}$\\
\bottomrule
\\
\multicolumn{6}{@{}l}{}
\end{tabular}
\end{table*}

This simplicity of this `three-sigma' rule can not be leveraged so easily in the proposed OGS algorithm.
However, we can still implement the concept of setting $\lam$ so as to reduce the noise down to a specified fraction of its original power.
For this purpose, the effect of the OGS algorithm on pure zero-mean Gaussian noise, $x(i) = \calN(0, \sigma^2)$, can be measured through computation.
In particular, the standard deviation of the OGS output as a function of $(\lam, K, \reg)$ can be found empirically and recorded.
For example, Table~\ref{table:lam} records the value
\begin{equation*}
	\alpha(\lam, K, \reg) = \frac{1}{\sigma} \std\bigl\{ \ogs(\x; \lam, K, \reg) \bigr\},
	\ x(i) = \calN(0, \sigma^2)
\end{equation*}
for several $ \lam $ and group sizes $ K $.
For this table we used the atan penalty function with $ a $ set to its maximum value of $1/(K\lam)$; i.e., 
$ \reg(\cdot) = \reg\satan(\cdot, 1/(K\lam)) $.
The value $ \alpha $ also depends on the number of iterations of the OGS algorithm.
In computing Table~\ref{table:lam} we have used a fixed number of 25 iterations.

We clarify how to use Table~\ref{table:lam} to set the regularization parameter:
Suppose in one-dimensional signal denoising, one seeks to set $ \lam $ so that
the OGS algorithm reduces $\sigma$ down to $ 10^{-4} \sigma $.
If one uses a group size of $ K = 5 $, the atan penalty function with $ a = 1/(5\lam) $, and 25 iterations,
then according to Table~\ref{table:lam}, one should use $ \lam = 1.2 \sigma $ (see the last column of the fifth row of the table).
For each group size $K$,
the table records a discrete set of $(\lam, \alpha)$ pairs for $ 10^{-4} < \alpha < 10^{-2} $.
Linear interpolation on a $\alpha$-logarithmic scale can be used to estimate $ \lam $ for other $\alpha$.
For example, if one seeks to set $\lam$ so that the OGS algorithm reduces $\sigma$
down to $ 10^{-3} \sigma $, then according to the interpolation illustrated in Fig.~\ref{fig:interp}, 
one should use $ \lam = 1.07 \sigma $.

\begin{figure}
	\centering
	\includegraphics{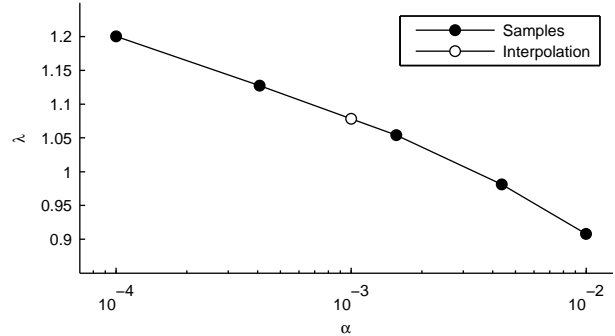}
	\caption{
		Solid dots indicate the values from Table~\ref{table:lam} for the group size $K=5$.
		The circle indicates the interpolated value at $\alpha = 10^{-3}$.
		\label{fig:interp}
	}
\end{figure}

To set $ \lam $ by this approach for other penalty functions, other values of $ a $,
and for complex data, 
it is necessary to compute additional tables.
We have precomputed a set of such tables to be available as supplementary material.
Using precomputed tables and interpolation, a suitable value for $ \lam $ can be found very quickly.
These tables assume the noise is AWGN; for other noise models, other tables need to be precomputed.
This approach is also effective for two-dimensional denoising (e.g., spectrogram denoising).

\smallskip
\noindent
\textbf{Monte-Carlo SURE.}
Another approach to select the regularization parameter, $\lam$, is based on minimizing the mean square error (MSE).
For the problem 
of denoising a signal in AWGN, the MSE is unknown in practice, due to the 
noise-free signal being unknown.
But, the MSE can be estimated using Stein's unbiased risk estimator (SURE) \cite{Stein1981a}.
To estimate the MSE, SURE requires only the observation $\y$, noise variance $\sigma^2$, and divergence of the estimator.
However, the computation of the divergence is intractable for many estimators, including OGS.
To overcome this issue, it is proposed in Ref.~\cite{Ramani_MonteCarloSure}
that Monte-Carlo methods be used. 
We have applied this approach, i.e., `Monte-Carlo SURE' (MC-SURE), to estimate the MSE
for complex-valued speech spectrogram denoising using OGS.
Since the spectrogram is complex, we calculate the MS-SURE MSE by averaging real and imaginary divergences as in \cite{Candes_SURE}.
Figure~\ref{fig:estimate_mse} illustrates both the MSE, as calculated by MC-SURE, and the true MSE,
as functions of $ \lam $.
The estimated MSE is quite accurate, and the MSE-optimal value of $ \lam $ is about 0.33.
However, a disadvantage of MC-SURE is its high computational complexity.
It requires two OGS optimizations for each $\lam$ to emulate the divergence.

It is noted in Ref.~\cite{Ramani_MonteCarloSure} that for non-smooth estimators, 
the MSE, as calculated by MC-SURE, tends to deviate randomly from the true MSE
(see Fig.~4 in \cite{Ramani_MonteCarloSure}).
For OGS, the MSE calculated by MC-SURE closely follows the true MSE, as illustrated in Fig~\ref{fig:estimate_mse},
which shows that OGS is close to continuous and bounded.

\begin{figure}
	\centering
	\includegraphics[width = 3.2in]{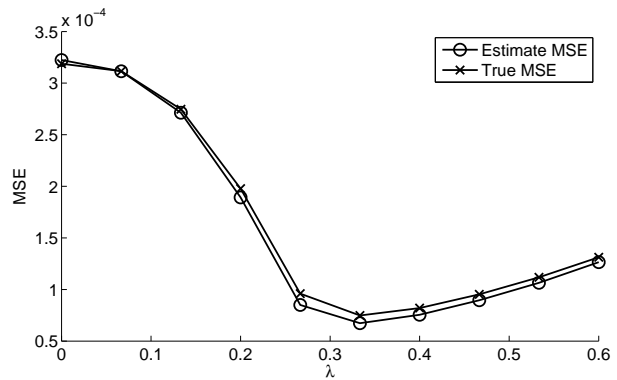}
	\caption{True MSE and  MSE calculated using Monte-Carlo SURE.
	\label{fig:estimate_mse}
	}
\end{figure}

\section{Experimental Results}
\label{sec:examples}

\subsection{Example 1: One-dimensional Signal Denoising}
\label{sec:Ex1}

This example compares the proposed non-convex regularized OGS algorithm
with the prior (convex regularized) version of OGS
and
with scalar thresholding.
The SNRs are summarized in Table~\ref{tab:snr_ogs1}.

Figure~\ref{fig:1Dsignals}a shows a synthetic group-sparse signal (same as in \cite{Chen_2014_OGS}).
The noisy signal, shown in Fig.~\ref{fig:1Dsignals}b, was obtained by adding white Gaussian noise (AWGN) with SNR of 10 dB.
For each of soft and hard thresholding, we used the threshold, $ T $, that maximizes the SNR. 
The SNR values are summarized in the top row of Table~\ref{tab:snr_ogs1}.

\begin{table}
\caption{
	Example 1. Output SNR
	\label{tab:snr_ogs1}
}
\centering
\begin{tabular}{@{}lrrrrr@{}}
\toprule
	 & \multicolumn{5}{c}{Estimator} \\ 
	 \cmidrule(l){2-6}
	Param.   & Hard thr.    	& Soft thr.	&	OGS[abs] &	OGS[log]	&	OGS[atan]
\\
 \midrule
max SNR     & 13.84   	& 12.17	&	12.30 &	14.52	&	15.37  
\\
$10^{-2}\sigma$   &  6.74     	& 3.86	&	8.01 &	 12.07	&	13.92   
\\
$10^{-3}\sigma$   &  5.05     	& 2.17	&	6.23 &	9.69	&	11.54   
\\
\bottomrule
\\[-0.5em]
\multicolumn{6}{@{}l}{SNR is in dB; $\sigma$ is the noise standard deviation.}

\end{tabular}
\end{table}

The result obtained using the prior version of OGS \cite{Chen_2014_OGS}  is shown in Fig~\ref{fig:1Dsignals}c.
This is equivalent to setting $\reg$ to the absolute value function; i.e. $\reg(x) = \abs{x}$.
So, we denote this as OGS[abs].
The result using the proposed non-convex regularized OGS is shown in Fig.~\ref{fig:1Dsignals}d.
We use the arctangent penalty function with $ a $ set to the maximum value of $1/(K\lam)$ that preserves convexity of $ F $;
i.e., we use $\reg(\cdot) = \reg\satan(\cdot, 1/(K\lam))$.
We denote this as OGS[atan].
We also used the logarithmic penalty (not shown in the figure).
For each version of OGS, we used a group size of $ K = 5 $, and we set $\lam$ to maximize the SNR.

Comparing soft thresholding and OGS[abs] (both of which are based on convex regularization),
it can be observed that OGS[abs] gives a higher SNR, but only marginally. 
Both methods leave residual noise,
as can be observed for OGS[abs] in Fig.~\ref{fig:1Dsignals}c.
On the other hand,
comparing OGS[atan] and OGS[abs], it can be observed that OGS[atan] (based on non-convex regularization)
is substantially superior: it has a substantially higher SNR and almost no residual noise is visible 
in the denoised signal.
Comparing OGS[log] and OGS[atan] with hard thresholding (see Table~\ref{tab:snr_ogs1}),
it can be observed the new non-convex regularized OGS algorithm also yields higher SNR than hard thresholding.
This example demonstrates the effectiveness of non-convex regularization for
promoting group sparsity.

To more clearly compare the result of OGS[abs] and OGS[atan], 
these two results are shown together in Fig.~\ref{fig:1Dio}.
In Fig.~\ref{fig:1Dio}a,
the output value, $x(i)$, is shown versus the input value, $y(i)$, for $ i \in \ZZ_N$.
Compared to OGS[abs], the OGS[atan] algorithm better preserves the amplitude of the non-zero values
of the original signal, 
while better thresholding small values.
Figure~\ref{fig:1Dio}b shows the denoising error for the two OGS methods.
It can be observed that the denoised signal produced by OGS[atan] has much less error than OGS[abs].
(For OGS[atan], the error is essentially zero for 50\% of the signal values.)

As a second experiment,  we selected $ T $ and $ \lam $ for each method, so as to reduce
the noise standard deviation, $ \sigma $, down to $0.01 \sigma$,
as described in Sec.~\ref{sec:setlam}.
The resulting SNRs, given in the second row of Table~\ref{tab:snr_ogs1}, are much lower.
(This method does not maximize SNR, but it does ensure residual noise
is reduced to the specified level.) 
The low SNR in these cases is due to the attenuation (bias) of large magnitude values. 
However, it can be observed that OGS, especially with non-convex regularization, 
significantly outperforms scalar thresholding.

\begin{figure}
	\centering
	\includegraphics{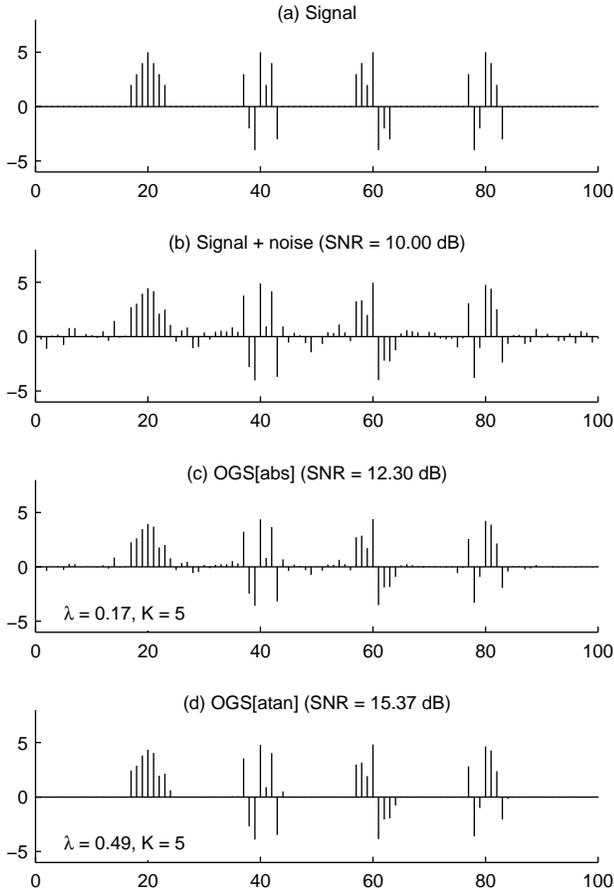}
	\caption{
	Example 1: Group-sparse signal denoising.
	\label{fig:1Dsignals}
	}
\end{figure}

\begin{figure}
	\centering
	\includegraphics{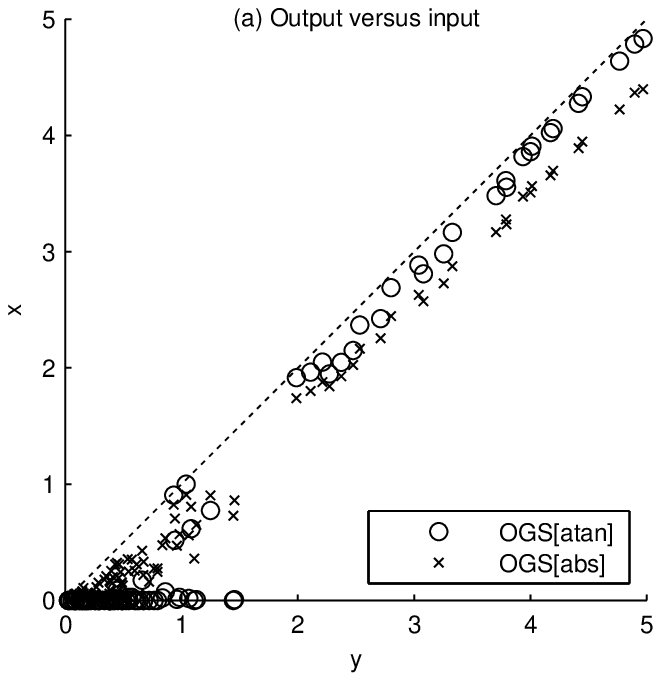}
	\includegraphics{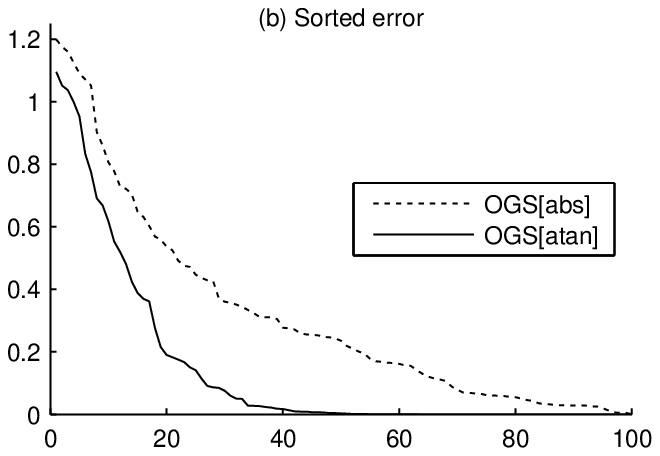}
	\caption{
	Example 1.
	Comparison of OGS[abs] and OGS[atan] in Fig.~\ref{fig:1Dsignals}.
	\label{fig:1Dio}
	}
\end{figure}

\subsection{Example 2: Speech Denoising}

This example evaluates the use of the proposed OGS algorithm for the problem of speech enhancement (denoising).
We compare the OGS algorithm with several other algorithms.
For the evaluation, we use female and male speakers,
multiple sentences, 
two noise levels, and two sampling rates.

Let $ \s = \{s(n), \ n \in \ZZ_N\} $ denote the noisy speech waveform
and $ \y = \{y(i), \ i \in \ZZ_{N_1} \times \ZZ_{N_2} \} = \STFT\{\s\} $ denote the complex-valued short-time Fourier transform of $ \s $.
For speech enhancement, 
we apply the two-dimensional form of the OGS algorithm to $ \y $ and then compute the inverse STFT;
i.e.,
\[
	 \x = \STFT\inv\{ \ogs( \STFT\{\s\}; \lam, K, \reg ) \}
\]
with $ K = (K_1, K_2) $ where $ K_1 $ and $ K_2 $ are the spectral and temporal widths of the two-dimensional group.
We implement the STFT with 50\% frame overlap and a frame duration of 32 milliseconds
(e.g., 512 samples at sampling rate 16 kHz).

Throughout this example, 
we use the non-convex arctangent penalty function with $ a $ set to its maximum value of $ a= 1/(K_1 K_2 \lam) $.
In all cases, we use a fixed number of 25 iterations within the OGS algorithm.

Each sentence in the evaluation is spoken by both a male and a female speaker.
There are 15 sentences sampled at 8 kHz,
and 30 sentences sampled at 16 kHz.
The 8 kHz and 16 kHz signals were obtained from Ref.~\cite{Loizou_2007} and a Carnegie Mellon University (CMU) website, respectively.\footnote{The CMU files were  downloaded from \url{http://www.speech.cs.cmu.edu/cmu_arctic/cmu_us_bdl_arctic/wav} and  \url{http://www.speech.cs.cmu.edu/cmu_arctic/cmu_us_clb_arctic/wav}.
This evaluation used files \texttt{arctic\_a0001} - \texttt{arctic\_a0030}. } 
To simulate noisy speech, we added white Gaussian noise. 

The time-frequency spectrogram of a noisy speech signal (\texttt{arctic\_a0001}) with an SNR of 10 dB is
illustrated in Fig.~\ref{fig:spectrogram}a.
Figure~\ref{fig:spectrogram}b illustrates the result of OGS[atan] using group size $ K = (8, 2)$; 
i.e., eight spectral samples by two temporal samples.
It can be observed that noise is effectively suppressed while details are preserved.

Figure~\ref{fig:spectrogram_detail_ogs} compares the proposed OGS[atan] algorithm
with the prior version of OGS \cite{Chen_2014_OGS}, i.e., OGS[abs].
The figure shows a single frame of the denoised spectrograms, corresponding to $ t = 0.79 $ seconds.
The prior and proposed OGS algorithms are illustrated in parts (a) and (b) respectively.
In both (a) and (b), samples of the noise-free spectrogram, to be recovered, are indicated by dots.
(The noisy spectrogram is not illustrated).
Comparing (a) and (b), it can be observed that above 2~kHz, OGS[atan] estimates the noise-free spectrum 
more accurately than OGS[abs].

In terms of run-time, for a signal of length $ N = 51761 $ (i.e., 3.2 seconds at sampling rate of 16 kHz), 
algorithms OGS[abs] and OGS[atan] ran in 0.18 and 0.22 seconds, respectively.
Timings were performed on a 2013 MacBook Pro (2.5 GHz Intel Core i5) running Matlab R2011a.

\smallskip
\noindent
\textbf{Regularization parameter.}
We have found empirically, that setting $\lam$ to maximize SNR
yields speech with noticeable undesirable perceptual artifacts (`musical noise').
This known phenomenon is due to residual noise in the STFT domain.
Therefore, we instead set the regularization parameter, $\lam$, using the noise suppression approach described in Sec.~\ref{sec:setlam}.
In particular, we set $ \lam $ so as to reduce the noise standard deviation $ \sigma $ down to $ (3 \times 10^{-4}) \sigma $.
We have selected this value so as to optimize the perceptual quality of the denoised speech
according to informal listening tests.
In particular, this value is effective at suppressing the `musical noise' artifact.
We also note that this approach leads to greater regularization (higher $ \lam $) than SNR-optimization of $\lam$.

\begin{figure}
	\centering
	\includegraphics{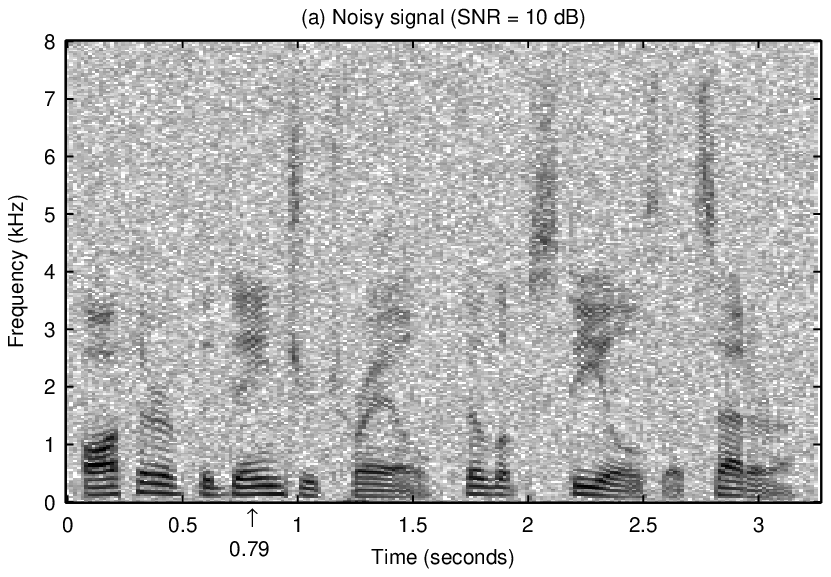}
	\includegraphics{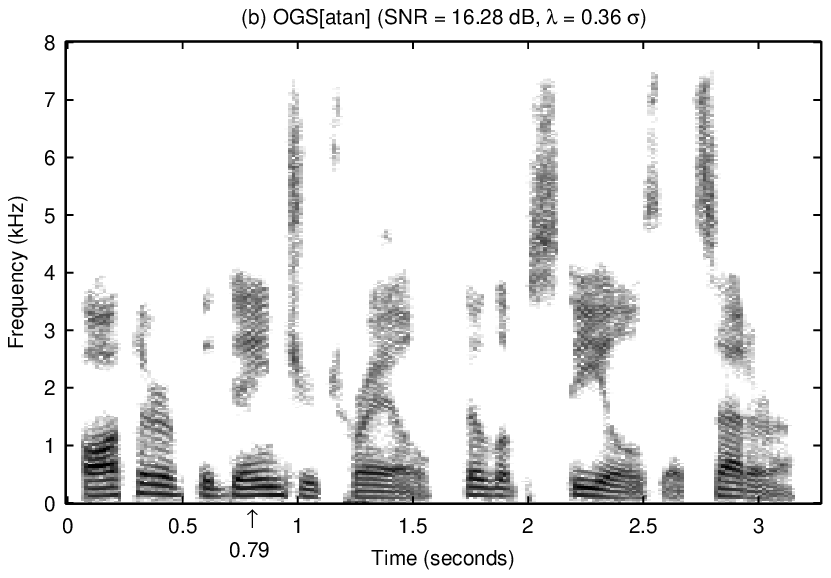}
	\caption{
	Spectrograms before and after denoising (male speaker).
	(a) Noisy signal.
	(b)  OGS[atan] with group size $K = (8, 2)$.
	Gray scale represents decibels.
	\label{fig:spectrogram}
	}
\end{figure}

\begin{figure}
\centering
\includegraphics{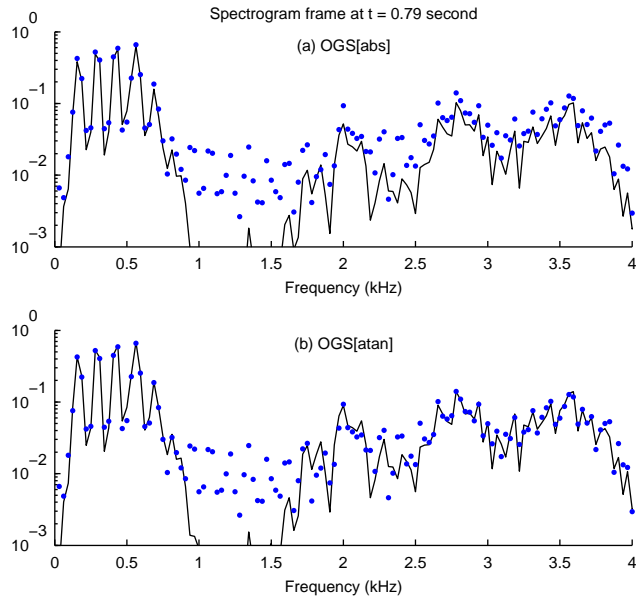}
\caption{
	Frequency spectrum of denoised spectrograms at $t = 0.79$ seconds.
	(a) OGS[abs]. (b) OGS[atan].
	The group size is $K = (8, 2)$ in both cases.
	The noise-free spectrum is indicated by dots.
	\label{fig:spectrogram_detail_ogs}
	}
\end{figure}

\smallskip
\noindent\textbf{Group size.} 
The perceptual quality of speech denoised using OGS depends on the specified group size.
As we apply OGS to a time-frequency spectrogram, 
the size of the group with respect to both the temporal and spectral dimensions 
must be specified. 
We let 
$K_1$ and $ K_2$ denote the number of spectral and temporal samples, respectively.

One approach to select the pair of parameters, $ (K_1, K_2) $,   
is to maximize the SNR for a set of denoising experiments.
We have performed OGS denoising for each of 30 noisy speech signals
using all pairs $ (K_1, K_2) $ such that
 $ 1 \leq K_1 \leq 10 $ and $ 1 \leq K_2 \leq 4 $.
In this experiment, we have used
speech sampled at 16 kHz, 
an SNR of 10 dB, and we have selected $\lam$ in each case
according to the preceding note [suppression of noise down to $(3\times10^{-4})\sigma$].
We found that for the male speaker, 
a group size of $(8, 2)$ maximized the SNR most frequently.
This conforms with our informal listening tests with different group sizes.
The denoised spectrum in Figure~\ref{fig:spectrogram}b was obtained using this group size of $(8, 2)$.

For the female speaker, the experiment reveals that a group size of $ (2, 4) $ maximizes the SNR most frequently.
However, we found that this group size results in poor perceptual quality.
To investigate the effect of group size, the denoised spectrograms using groups of size $(8,2)$ and $(2, 4)$
are illustrated in Fig.~\ref{fig:spectrogram_female}.
Fig.~\ref{fig:spectrogram_female}a shows the noisy spectrogram (file \texttt{arctic\_a0001}).
We highlight two areas of the spectrogram. 
The low-frequency area, denoted `A', exhibits a high level of temporal correlation.
On the other hand,
the high-frequency area, denoted `B', exhibits a high level of spectral correlation.
Figs.~\ref{fig:spectrogram_female}(b,c) show areas A and B of the spectrogram obtained using group size $(8, 2)$.
Figs.~\ref{fig:spectrogram_female}(d,e) show areas A and B of the spectrogram obtained using group size $(2, 4)$.

\begin{figure}
\centering
\includegraphics{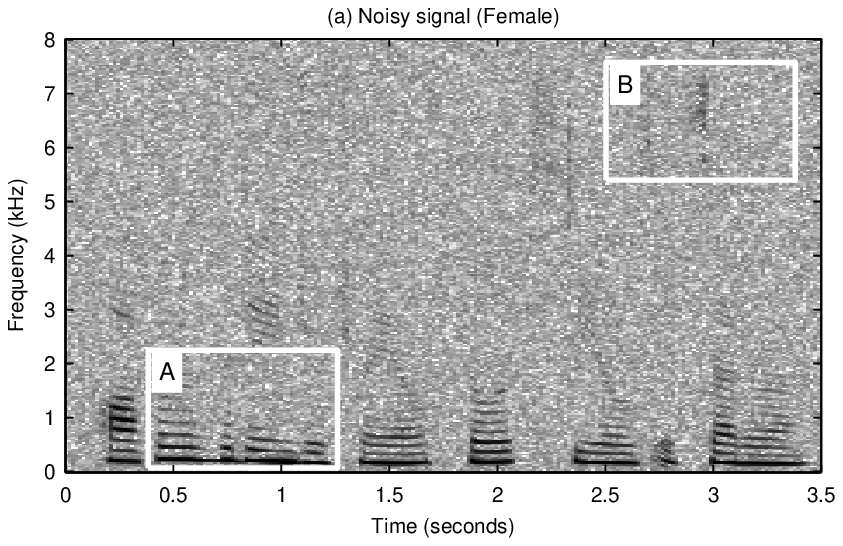}
\includegraphics{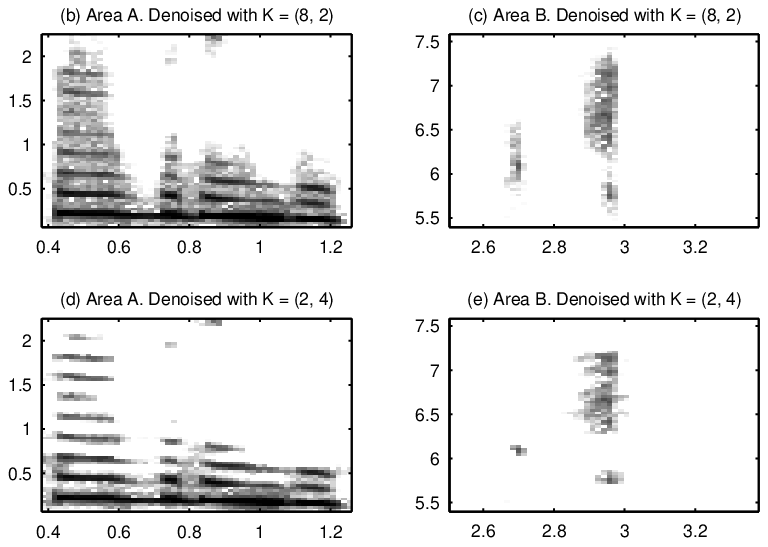}
\caption{
	Denoised spectrograms; female speaker. 
	(a) Noisy spectrogram with SNR = 10 dB.
	(b, c) Areas A and B, denoised with group size $ (8, 2)$.
	(d, e) Areas A and B, denoised with group size $ (2, 4)$.
	\label{fig:spectrogram_female}
	}
\end{figure}

It can be observed in area A that group size $(2, 4)$ suppresses the inter-formant noise more
completely than group size $(8, 2)$.
Conversely, in area B, group size $(8, 2)$ recovers the original spectrogram more accurately than group size $(2,4)$.
Since area A is representative of more of the spectrogram than area B, 
the SNR-optimal group size for the whole spectrogram is $(2, 4)$.
However, due to the distortion of high frequencies, as in area B, group size $(2, 4)$ yields 
the perceptually inferior result.
Moreover, the lower inter-formant noise suppression of group size $(8, 2)$ appears to have a negligible adverse impact
on perceptual quality.
Therefore, even though group size $(2,4)$ yields a higher SNR for the female speaker, 
we use group size $(8,2)$ in the evaluation of OGS due to its superior perceptual quality.
This points to the potential value of allowing groups in OGS to be sized adaptively, as in Ref.~\cite{Yu_2008_TSP}.
However, we do not explore such an extension of OGS in this work.

We conducted equivalent evaluations at the sampling rate of 8 kHZ in order to determine
an appropriate group size for this case.
We found that 
group sizes of  $K = (7, 2)$ and $K = (3, 3)$ were optimal in terms of SNR, for the male and female speaker, respectively.
As above, we selected the group size $K = (7, 2)$ for both genders for its better perceptual quality.

\smallskip
\noindent
\textbf{Algorithm comparisons.}
In Table~\ref{table:sealgs} we compare the OGS[atan] algorithm with several other speech enhancement algorithms.
The table summarizes the output SNR for 
two sampling rates, male and female speakers, and two input SNR (noise) levels.
Each SNR value is averaged over 30 or 15 sentences,
depending on the sampling rate.
It can be observed that the proposed algorithm, OGS[atan], achieves the highest SNR in each case.
(We also note that in all cases, OGS is used not with SNR-optimized $\lam$, 
but with the larger $\lam$, set according to the noise suppression method. 
The SNR of OGS could be further increased, but at the cost of perceptual quality.)

The algorithms used in the comparison are:
spectral subtraction (SS) \cite{Berouti_1979},
the log-MMSE algorithm (LMA)  \cite{Cohen2002},
the subspace algorithm (SUB) \cite{Hu2003},
block thresholding (BT) \cite{Yu_2008_TSP},
and persistent shrinkage (PS) \cite{Siedenburg_2013_JAES}.
For SS, LMA, and SUB, we used the MATLAB software provided in Ref.~\cite{Loizou_2007}.
For the BT\footnote{\url{http://www.cmap.polytechnique.fr/~yu/research/ABT/samples.html}}
and PS\footnote{\url{http://homepage.univie.ac.at/monika.doerfler/StrucAudio.html}}
algorithms,
we used the software provided by the authors on their web pages.

Furthermore, we additionally evaluated each method with empirical Wiener post-processing (EWP) \cite{Ghael97improvedwavelet}.
The EWP technique is based on mean square error minimization
and 
its effectiveness has been well demonstrated \cite{Dabov_2007_TIP,Yu_2008_TSP,Chen_2014_OGS}.
In Table~\ref{table:sealgs}, SNR values obtained using EWP are shown in parenthesis
for each algorithm and scenario. 

\begin{table}
	\centering
	\caption{
	\label{table:sealgs}
	Average SNR for six speech enhancement algorithms.
	} 
	\begin{tabular}{@{}lcccc@{}}
	    \multicolumn{5}{c}{ (a) $f_s = 16$ kHz (average of 30 samples) \vspace*{0.3em} }
	    \\
		\toprule
		 & \multicolumn{2}{c}{Male / Input SNR (dB)} & \multicolumn{2}{c}{Female / Input SNR (dB)}  \\
		    \cmidrule(lr){2-3}
		    \cmidrule(l){4-5}
		Method                                         &    5     &    10                 & 5       & 10              
		\\
		\midrule
		SS     &  9.44\,(10.96)     & 13.63\,(14.99)  & 13.36\,(14.59)     & 16.86\,(17.93)   
		\\
		LMA           & 10.24\,(11.64)   & 13.30\,(15.25)   & 13.30\,(15.16)   & 15.71\,(18.13)  
		\\
		SUB          & 11.28\,(12.31)   & 13.94\,(16.11) & 13.39\,(15.31)   & 15.05\,(18.48)
		\\
		BT          & 12.00\,(12.49)   & 15.61\,(16.10)    & 15.09\,(15.69)   & 18.18\,(18.78)     
		\\
		PS          & 10.75\,(12.00)   & 14.17\,(15.73)  & 12.67\,(14.71)   & 16.39\,(18.13)  
		\\
		OGS[abs]         & 10.48\,(12.36)   & 13.92\,(16.00)     & 12.91\,(15.53)   & 16.24\,(18.60)      
		\\
		OGS[atan]                             & 12.93\,(12.98)   & 16.58\,(16.58)  & 15.37\,(15.83)   & 18.68\,(19.02)
		\\
		\bottomrule
		\\
	\end{tabular}
	
	\begin{tabular}{@{}lcccc@{}}
	    \multicolumn{5}{c}{ (b) $f_s = 8$ kHz (average of 15 samples)\vspace*{0.3em}  }
	    \\
		\toprule
		 & \multicolumn{2}{c}{Male / Input SNR (dB)} & \multicolumn{2}{c}{Female / Input SNR (dB)}  \\
		    \cmidrule(lr){2-3}
		    \cmidrule(l){4-5}
		Method       &    5             &    10            & 5                & 10              
		\\
		\midrule
		SS           & 10.73\,(11.75)   & 14.57\,(15.54)   & 10.45\,(11.59)   & 14.38\,(15.47)   
		\\
		LMA          & 10.66\,(12.00)   & 13.75\,(15.61)   &  9.34\,(11.05)   & 12.51\,(14.85)  
		\\
		SUB          & 10.83\,(12.29)   & 14.03\,(16.06)   &  9.57\,(11.53)   & 13.25\,(15.55)
		\\
		BT           & 11.80\,(12.48)   & 15.45\,(16.10)   & 11.54\,(12.40)   & 15.12\,(16.00)     
		\\
		PS           & 10.45\,(12.20)   & 13.64\,(15.75)   &  9.11\,(11.20)   & 13.52\,(15.47)  
		\\
		OGS[abs]          &  9.96\,(12.25)   & 13.42\,(15.87)   &  9.34\,(11.91)   & 12.81\,(15.70)      
		\\
		OGS[atan]     & 12.80\,(12.97)   & 16.41\,(16.53)   & 12.10\,(12.62)   & 15.84\,(16.31)
		\\
		\bottomrule
	\end{tabular}
\end{table}

The proposed algorithm, OGS[atan], achieves the highest SNR for both noise levels and genders.
For example, for the male speaker with an input SNR of 10 dB, OGS[atan] attains the
highest output SNR of 16.58 dB.
BT achieves the second highest, 15.61 dB.
In terms of perceptual quality,
SS and LMA have clearly audible artifacts;
BT and PS have slight audible artifacts;
OGS[atan], OGS and SUB have the least audible artifacts.
However, SUB has a high computational complexity due to eigenvalue factorization.
Compared to OGS[abs] and SUB, OGS[atan] better preserves the perceptual quality of high frequencies.
Similar results can be observed for different noise levels and the female speaker.

Empirical Wiener post-processing (EWP) improves the SNR for all methods at all noise levels,
but least for OGS[atan].
EWP is effective for increasing SNR because it effectively rescales large STFT coefficients that
are unnecessarily attenuated by these algorithms (the results of which are biased toward zero).
The fact that EWP yields the least improvement for OGS[atan] demonstrates that
this algorithm inherently induces less bias than the other algorithms.
 
According to informal listening tests (conducted at input SNR of 10 dB, $f_s $ of 16 kHz), the effect of EWP on
audible artifacts depends on the algorithm.
Although EWP improves the SNR of SS and LMA, denoising artifacts are still clearly perceptible.
EWP improves the perceptual quality of BT and PS slightly.
EWP also improves perceptual quality of OGS[abs] and SUB, which already had good perceptual quality.
The effect of EWP on OGS[atan] is almost imperceptible; its good perceptual quality is maintained.

Figure~\ref{fig:AllSample_snr_male_fs16k} illustrates the individual SNRs of the 30 sentences
denoised using each of the utilized algorithms
(male, input SNR of 10 dB, $f_s$ of 16 kHz).
It can be observed that EWP improves each algorithm, except OGS[atan].
However, as shown in Fig.~\ref{fig:AllSample_snr_male_fs16k}b, 
OGS[atan] outperforms the other algorithms in terms of SNR irrespective of EWP.

\begin{figure}
	\centering
	\includegraphics[width = 3.5in]{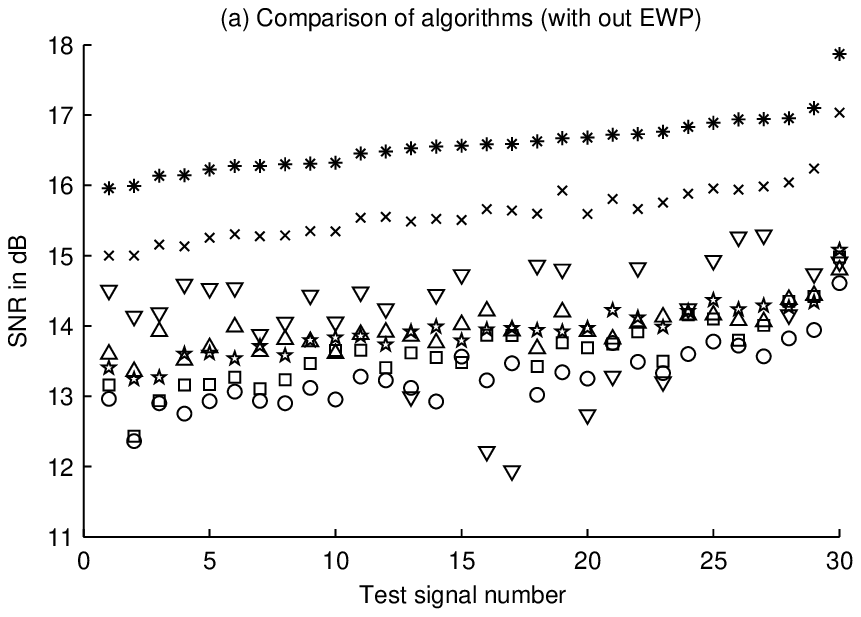}
	\includegraphics[width = 3.5in]{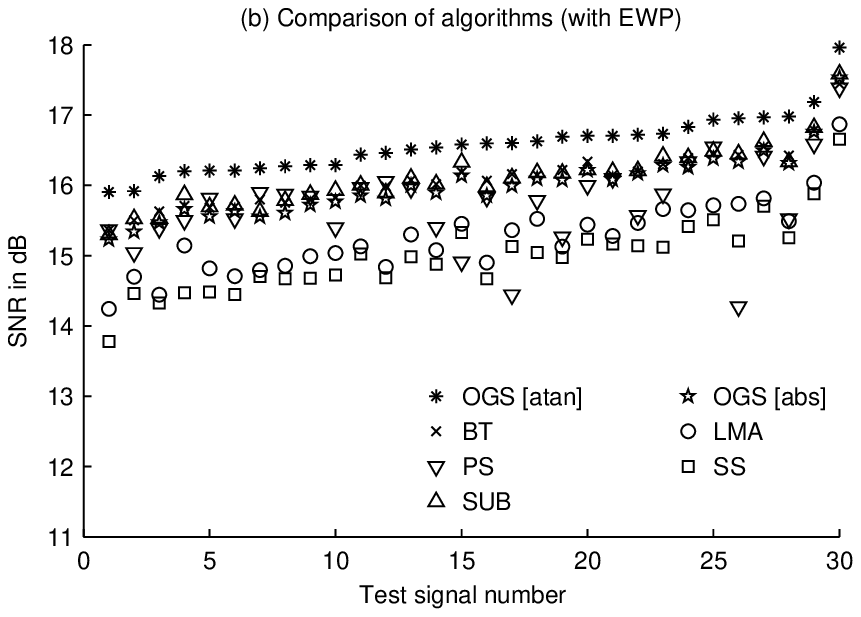}
	\caption{
	SNR comparison of speech enhancement algorithms (30 male sentences, input SNR of 10 dB).
	Each algorithm is used without EWP (a) and with EWP (b).
	The sentences are ordered according the SNR of OGS[atan].
	\label{fig:AllSample_snr_male_fs16k}
	}
\end{figure}

\section{Remarks}

Several aspects of the non-convex regularized OGS algorithm are sufficiently similar to 
those of the convex regularized OGS algorithm~\cite{Chen_2014_OGS} that we refer the 
reader to Ref.~\cite{Chen_2014_OGS}.
In particular, remarks in Ref.~\cite{Chen_2014_OGS} regarding the convergence behavior,
implementation issues, computational complexity, and relationship of OGS to FOCUSS \cite{Rao_2003_TSP}, apply also to the version of OGS presented here.

The proximal framework has proven effective for 
convex optimization 
problems arising in sparse signal estimation and reconstruction  \cite{Combettes_2011_chap, Combettes_2008_SIAM}.
The proposed non-convex regularized OGS algorithm resembles a proximity operator;
however, a proximity operator is defined in terms of a convex penalty function \cite{Combettes_2011_chap}.
Hence, the proposed approach appears to fall outside the proximal framework.
Due to the effectiveness of the proximal framework for solving inverse problems
much more general than denoising (e.g.\ deconvolution),
it will be of interest in future work
to explore the extent to which the proposed
method can be used for more general inverse problems by using proximal-like techniques.

\section{Conclusion}

This paper formulates group-sparse signal denoising as a convex optimization problem with a non-convex regularization term.
The regularizer is based on overlapping groups so as to promote group-sparsity.
The regularizer, being concave on the positive real line, promotes sparsity more strongly than any convex regularizer can.
For several non-convex penalty functions, parameterized by a variable, $a$, it has been shown how to constrain $ a $ to 
ensure the optimization problem is strictly convex.
Numerical experiments demonstrate the effectiveness of the proposed method for speech enhancement.

\appendix

The proof of Proposition \ref{prop:cvxcond}
relies on the following theorem and corollary.
\begin{thm}
\label{thm:convexity_book}
(Theorem 6.4, page 16, Ref.~\cite{Hiriart_2001})
Let a function $f$ be continuous on an open interval $I$ and possess an increasing right-derivative, or an increasing left-derivative, on $I$.
Then $f$ is convex on $I$.
\end{thm}
Note that  $f$ is \emph{strictly} convex if $f$ has either a monotone increasing right-derivative, or a monotone increasing left-derivative, on $I$.

\begin{cor}
\label{thm:convexity}
Suppose $G:\RR\rightarrow\RR$ is  continuous, and the second derivative of $G$ exists satisfying $G''(x)> 0$ on $\RNZ$.
If $G'(0^-) < G'(0^+)$, then $G$ is strictly convex on $\RR$.
\end{cor}
\begin{proof}
Based on Proposition~\ref{thm:convexity_book}, it is sufficient to prove that the right derivative of $G$ is monotone increasing on $\RR$.
For all $x < 0$, since $G''(x) > 0$,
we have $G'(x) = G'(x^+) = G'(x^-)$ is monotone increasing.
We also have $G'(x^+)$ is monotone increasing for $x>0$.
For any $x_1 < 0$ and $x_2>0$, we have $G'(x_1^+) = G'(x_1^-) <  G'(0^-) < G'(0^+) < G'(x_2^+)$.
If follows that $G'(x^+)$ is monotone increasing on $\RR$, and hence $G$ is strictly convex.
\end{proof}

\bibliographystyle{plain}

%
%
%

\end{document}